\documentclass[11pt]{article}
\pdfoutput=1
\usepackage{fullpage}
\usepackage{graphicx,color}

\usepackage{amstext,amssymb,amsmath}
\usepackage{amsthm}
\usepackage{mathtools}

\usepackage[numbers]{natbib}

\usepackage{array,float}
\usepackage[usenames,dvipsnames]{xcolor}

\usepackage{verbatim}
\usepackage{bm}
\usepackage{paralist}
\usepackage{ulem}

\usepackage[draft=false]{hyperref}
\usepackage[capitalize]{cleveref}
\usepackage[noend]{algorithmic}
\usepackage{algorithm}
\usepackage{xspace}
\usepackage{delimset}
\newif\ifnotes\notestrue

\ifnotes
\usepackage[textsize=scriptsize,textwidth=2cm]{todonotes}
\else
\usepackage[disable]{todonotes}
\fi

\newtheorem{lem}{Lemma}[section]

\newtheorem{thm}[lem]{Theorem}
\newtheorem{cor}[lem]{Corollary}

\newtheorem{defn}[lem]{Definition}

\newtheorem{proposition}[lem]{Proposition}

\crefname{thm}{Theorem}{Theorems}

\newcommand{\wt}[1]{\smash{\widetilde{#1}}}

\newcommand{\eps}{\varepsilon}
\newcommand{\E}{\mathbb{E}}

\DeclareMathOperator*{\argmin}{arg\,min}

\providecommand{\reals}{\mathbb{R}}
\renewcommand{\Re}{\mathbb{R}}
\newcommand{\Nat}{\mathbb{N}}

\newcommand{\N}{\mathcal{N}}

\newcommand{\Id}{\mathbb{I}}

\newcommand{\A}{\mathcal{A}}

\newcommand{\bE}{\mathbb{E}}

\newcommand{\cA}{\mathcal{A}}

\newcommand{\cK}{{\mathcal{K}}}

\newcommand{\cN}{\mathcal{N}}
\newcommand{\cO}{\mathcal{O}}
\newcommand{\cP}{\mathcal{P}}

\newcommand{\cX}{\mathcal{X}}

\newcommand{\cZ}{\mathcal{Z}}

\newcommand{\ed}{\ensuremath{(\eps,\delta)}}

\newcommand{\ud}{\mathrm d}

\newcommand{\normal}[2]{\mathcal{N}(#1, #2)}
\newcommand{\conv}{\ast}
\newcommand{\Renyi}{R\'enyi\,}

\DeclarePairedDelimiterX{\infdivx}[2]{(}{)}{#1\;\delimsize\|\;#2}
\newcommand{\Dalpha}[2]{\mathrm{D}_{\alpha}\infdivx*{#1}{#2}}

\newcommand{\eff}{\ensuremath{\psi}}
\newcommand{\efs}{\ensuremath{\{\psi_t\}}}
\newcommand{\efps}{\ensuremath{\{\psi'_t\}}}

\newcommand{\K}{\mathcal{K}}

\newcommand{\dist}{\ensuremath{\mathcal{D}}}
\newcommand{\dists}{\ensuremath{\{\dist_t\}}}
\newcommand{\notbeta}{\gamma}

\providecommand{\ignore}[1]{\relax}

\newcommand{\CNI}{\ensuremath{\mbox{CNI}}\xspace}

\providecommand{\equ}[1]{

\begin{equation}
#1
\end{equation}}

\providecommand{\ignore}[1]{\relax}

\begin{document}
\title{Private Stochastic Convex Optimization:\\ Optimal Rates in Linear Time}
\author{%
    Vitaly Feldman\thanks{Apple; work done while at Google Research; \texttt{vitaly.edu@gmail.com}.}
    \and Tomer Koren\thanks{School of Computer Science, Tel Aviv University, and Google; \texttt{tkoren@tauex.tau.ac.il}.}
    \and Kunal Talwar\thanks{Apple; work done while at Google Research; \texttt{kunal@kunaltalwar.org}.}
    }

\maketitle

\begin{abstract}
We study differentially private (DP) algorithms for {\em stochastic convex
optimization}: the problem of minimizing the population loss given
i.i.d.~samples from a distribution over convex loss functions.  A recent work of
Bassily et al.~(2019) has established the optimal bound on the excess population
loss achievable given $n$ samples. Unfortunately, their algorithm achieving this
bound is relatively inefficient: it requires $O(\min\{n^{3/2}, n^{5/2}/d\})$
gradient computations, where $d$ is the dimension of the optimization problem.

We describe two new techniques for deriving DP convex optimization algorithms
both achieving the optimal bound on excess loss and using $O(\min\{n, n^2/d\})$
gradient computations. In particular, the algorithms match the running time of
the optimal non-private algorithms. The first approach relies on the use of
variable batch sizes and is analyzed using the privacy amplification by
iteration technique of Feldman et al.~(2018). The second approach is based on a
general reduction to the problem of localizing an approximately optimal solution
with differential privacy. Such localization, in turn, can be achieved using
existing (non-private) uniformly stable optimization algorithms.  As in the
earlier work, our algorithms require a mild smoothness assumption.  We also give
a linear-time algorithm achieving the optimal bound on the excess loss for the
strongly convex case, as well as a faster algorithm for the non-smooth case.
\end{abstract}

\thispagestyle{empty}
\pagebreak
\setcounter{page}{1}

\section{Introduction}

Stochastic convex optimization (SCO) is the problem of minimizing the expected loss (also referred to as {\em population loss}) $F(w) = \E_{x \sim \cP}[f(w,x)]$ for convex loss functions of $w$ over some $d$-dimensional convex body $\K$ given access to i.i.d.~samples $x_1, \ldots, x_n$ from the data distribution $\cP$. The performance of an algorithm for the problem is measured by bounding the {\em excess (population) loss} of a solution $w$, that is the value $F(w) - \min_{v\in \K} F(v)$. This problem is central to numerous applications in machine learning and arises for example in least squares/logistic regression, or minimizing a convex surrogate loss for a classification problem. It also serves as the basis for the development of continuous optimization algorithms in the non-convex setting. In this work we study this problem with the constraint of differential privacy with respect to the set of samples \cite{DMNS06}.

Placing a differential privacy constraint usually comes at a cost in terms of utility. In this case, it is measured by the excess population loss of the solution, for a given number of samples $n$. Additionally, runtime efficiency of an optimization method is crucial for modern applications on large high-dimensional datasets, and this is the primary reason for the popularity of stochastic gradient descent-based methods. This motivates the problem of understanding the trade-offs between computational efficiency, and excess population loss in the presence of privacy constraints.

Differentially private convex optimization is one of most well-studied problems in private data analysis \cite{CM08,chaudhuri2011differentially,jain2012differentially,kifer2012private,ST13sparse,song2013stochastic,DuchiJW13,ullman2015private,JTOpt13,bassily2014differentially,talwar2015nearly,smith2017interaction,wu2017bolt,wang2017differentially,iyengartowards}.
However, most of the prior work focuses on the easier problem of minimizing the empirical loss $\hat F(w) = \frac{1}{n} \sum_i f(w,x_i)$ (referred to as empirical risk minimization (ERM)) for which tight upper and lower bounds on the excess loss are known in a variety of settings.
Upper bounds for the differentially private ERM can be translated to upper bounds on the population loss by appealing to {\em uniform convergence} of empirical loss to population loss, namely an upper bound on $\sup_{w\in \K} (F(w) - \hat F(w))$.
However, in general,\footnote{At the same time, uniform convergence suffices to derive optimal bounds on the excess population loss in a number of special cases, such as regression for generalized linear models.} this approach leads to suboptimal bounds: it is known that there exist distributions over loss functions over $\Re^d$ for which the best bound on uniform convergence is $\Omega(\sqrt{d/n})$ \cite{feldman2016generalization}. As a result, in the high-dimensional settings often considered in modern ML (when $n = \Theta(d)$), bounds based on uniform convergence are $\Omega(1)$ and do not lead to meaningful bounds on population loss.

The first work to address the population loss for SCO with differential privacy (DP-SCO) is \cite{bassily2014differentially} who give a bound of order $\max\{d^{1/4}\big/\sqrt{n}, \eps^{-1}\sqrt{d}\big/n\}$ \cite[Sec. F]{bassily2014differentially}.\footnote{For clarity, in the introduction we focus on the dependence on $d$, $n$ and $\eps$ for $(\eps,\delta)$-DP, suppressing the dependence on $\delta$ and on parameters of the loss function such as Lipschitz constant and the diameter of $\K$.}
For the most relevant case where $d=\Theta(n)$ and $\eps=\Theta(1)$, this results in a bound of $\Omega(n^{-1/4})$ on excess population loss.
More recent work of Bassily et al.~\cite{BassilyFTT19} demonstrates the existence of an efficient algorithm that achieves a bound of $O(1\big/\sqrt{n} + \eps^{-1}\sqrt{d} \big/ n)$, which is also shown to be tight.
Notably, this bound is comparable to the non-private SCO bound of $O(1/\sqrt{n})$ as long as $d/\eps^2=O(n)$. Their algorithm is based on solving the ERM via noisy stochastic gradient descent (SGD) \cite{bassily2014differentially} but requires relatively large batch sizes for the privacy analysis. As a result, their algorithm uses $O(\min\{n^{3/2}, n^{5/2}\big/d\})$ gradient computations. This is substantially less efficient than the optimal non-private algorithms for the problem which require only $n$ gradient evaluations.
They also give a near-linear-time algorithm under an additional strong assumption that the Hessian of each loss function is rank-1 over the entire domain.

Along the other axis, several of the aforementioned works on private ERM~\cite{wu2017bolt, wang2017differentially, iyengartowards} are geared towards finding computationally efficient algorithms for the problem, often at the cost of worse utility bounds.

We describe two new techniques for deriving linear-time algorithms that achieve the (asymptotically) optimal bounds on the excess population loss. Thus our results show that for the problem of Stochastic Convex Optimization, under mild assumptions, a privacy constraint come for free. For $d \leq n$, there is no overhead in terms of either excess loss or the computational efficiency. When $d \geq n$, the excess loss provably increases, but the optimal bounds can still be achieved without any computational overhead. Unlike the earlier algorithm \cite{BassilyFTT19} that solves the ERM and relies on uniform stability of the algorithm to ensure generalization, our algorithms directly optimize the population loss.

Formally, our algorithms satisfy the following bounds:%
\begin{thm}
\label{thm:main-intro}
Let $\K \subseteq \Re^d$ be a convex set of diameter $D$ and $\{f(\cdot,x)\}_{x\in \cX}$ be a family of convex $L$-Lipschitz and $\beta$-smooth functions over $\K$. For every $\rho >0$, there exists an algorithm $\A$ that given
a starting point $w_0\in \K$, and $S\in \cX^n$ returns a point $\hat w$. For all $\alpha \geq 1$, $\A$ uses $n$ evaluations of the gradient of $f(w,x)$ and satisfies $\left(\alpha, \alpha \rho^2/2 \right)$-RDP as long as $\beta \leq c\frac{L}{D} \min(\sqrt{n},\rho n/\sqrt{d})$, where $c$ is a universal constant. Further, if $S$ consists of samples drawn i.i.d.~from a distribution $\cP$ over $\cX$, then
\[
\E[F(\hat w)] \leq F^* + O\left( D L \cdot \left(\frac{1}{\sqrt{n}} + \frac{\sqrt{d}}{\rho n} \right)\right),
\]
where, for all $w\in \K$, $F(w) \doteq \E_{x \sim \cP}[f(w,x)]$, $F^* \doteq \min_{w\in \K} F(w)$ and the expectation is taken over the random choice of $S$ and randomness of $\A$.
\end{thm}
Our guarantees are stated in terms of \Renyi differential privacy (RDP) \cite{mironov2017renyi} for all orders $\alpha$ and can also be equivalently stated as 0-mean $(\rho^2/2)$-concentrated differential privacy (or $(\rho^2/2)$-zCDP) \cite{BS16-zCDP}. Standard properties of RDP/zCDP imply that our algorithms satisfy $(2\rho \sqrt{\ln(1/\delta)},\delta)$-DP for all $\delta > 0$ as long as $\rho \leq \sqrt{\ln(1/\delta)}$. Thus for $(\eps,\delta)$-DP our bound is
\[
\E[F(\hat w)] \leq F^* + O\left( D L \cdot \left(\frac{1}{\sqrt{n}} + \frac{\sqrt{d\ln(1/\delta)}}{\eps n} \right)\right),
\]
matching the tight bound in \cite{BassilyFTT19}.
We now overview the key ideas and tools used in these techniques.
\paragraph{Snowball-SGD:}
Our first algorithm relies on a one-pass noisy SGD with gradually growing batch sizes. Namely, at step $t$ out of $T$ the batch size is proportional to $1/\sqrt{T-t+1}$. We refer to SGD with such schedule of batch size as {\em Snowball-SGD}. The analysis of this algorithm relies on two tools. The first one is privacy amplification by iteration \cite{FeldmanMTT18}. This privacy amplification technique ensures that for the purposes of analyzing the privacy guarantees of a point $x_i$ used at step $t$ one can effectively treat all the noise added at subsequent steps as also added to the gradient of the loss at $x_i$. A direct application of this technique to noisy SGD results in different privacy guarantees for different points \cite{FeldmanMTT18} and, as a result, the points used in the last $o(n)$ steps will not have sufficient privacy guarantees. However, we show that by increasing the batch size in those steps we can achieve the optimal privacy guarantees for all the points.

A limitation of relying on this analysis technique is that the privacy guarantees apply only to the algorithm that outputs the last iterate of SGD. In contrast, the optimization guarantees usually apply to the average of all the iterates (see Section \ref{sec:no-average-privacy} for an example in which the privacy guarantees for the average iterate are much worse than those for the last iterate). Thus the second tool we rely on is the recent work of Jain et al.~\cite{JainNN19} showing that, for an appropriate choice of step sizes in SGD, the last iterate has the (asymptotically) optimal excess population loss. (Without the special step sizes the last iterate has excess loss larger by a $\log n$ factor \cite{shamir2013stochastic,HarveyLPR19}.) See Section \ref{sec:batches} for additional details of this approach.

\paragraph{No Privacy Amplification for the Average Iterate:}
It is natural to ask if the last iterate analysis is really needed or if the average iterate itself be proven to have good privacy properties. In \cref{sec:no-average-privacy}, we address this question and show that in general the average iterate can be very non-private even when the noise is sufficient to give strong privacy guarantees for the last iterate.

\paragraph{Localization:}
Our second approach is based on an (implicit) reduction to an easier problem of localizing an approximate minimizer of the population loss. Specifically, the reduction is to a differentially private algorithm that given a point $w_0$ that is within distance $R$ from the minimizer of the loss, finds a point $\hat w$ that is within distance $R/2$ from a point that approximately minimizes the loss. By iteratively using such a localizing algorithm with appropriately chosen parameters, a sufficiently good solution will be found after a logarithmic number of applications of the algorithm. Each application operates on its own subset of the dataset and thus this reduction preserves the privacy guarantees of the localizing algorithm.

A simple way to implement a localization algorithm is to start with non-private SCO algorithm whose output has optimal $L_2$ sensitivity. Namely, solutions produced by the algorithm on any two datasets that differ in one point are at distance on the order of $R/\sqrt{n}$ (this property is also referred to as {\em uniform stability} in the parameter space). Given such an algorithm one can simply add Gaussian noise to the output. This is a standard approach to differentially private optimization referred to as output perturbation~\cite{chaudhuri2011differentially,wu2017bolt}. However, for the purposes of localization, we only need to be within $R/2$ of the solution output by the algorithm and so we can add much more noise than in the standard applications, thereby getting substantially better privacy guarantees.

We note that in order to ensure that the addition of Gaussian noise localizes the solution with probability at least $1-\alpha$ we would need to increase the noise variance by an additional $\ln(1/\alpha)$ factor making the resulting rate suboptimal by a logarithmic factor. Thus, instead we rely on the fact that for algorithms based on SGD the bound on excess loss can be stated in terms of the second moment of the distance to the optimum.

We can now plug in existing uniformly stable algorithms for SCO. Specifically, it is known that under mild smoothness assumptions, one-pass SGD finds a solution that both achieves optimal bounds on the excess population loss and stability \cite{hardt2015train,feldman2019high}. This leads to the second algorithm satisfying the guarantees in Theorem~\ref{thm:main-intro}. See Section \ref{sec:localize} for additional details of this approach.

\paragraph{Non-smooth case:}
Both of our algorithms require essentially the same and relatively mild smoothness assumption: namely that the smoothness parameter is at most $\sqrt{n}$ (ignoring the scaling with $D,L$ and for simplicity focusing on the case when $d=O(n)$ and $\eps = 1$).
Bassily et al.~\cite{BassilyFTT19} show that optimal rates are still achievable even without this smoothness assumption. Their algorithm for the problem relies on using the prox operator instead of gradient steps which is known to be equivalent to gradient steps on the loss function smoothed via the Moreau-Yosida envelope. Unfortunately, computing the prox step with sufficient accuracy requires many gradient computations and very high accuracy is needed due to potential error accumulation. As a result, implementing the algorithm in \cite{BassilyFTT19} requires $O(n^{4.5})$ gradient computations.
Our reduction based technique gives an alternative and simpler way to deal with the non-smooth case. One can simply plug in a uniformly stable algorithms for SCO in the non-smooth case from \cite{shalev2010learnability}. This algorithm relies on solving ERM with an added strongly convex $\lambda \|w\|_2^2$ term. In this case the analysis of the accuracy to which the ERM needs to be solved is straightforward. However achieving such accuracy with high probability requires $O(n^2)$ gradient computations thus giving an $O(n^2)$ algorithm for the non-smooth version of our problem. Improving this running time is a natural avenue for future work. We remark that finding a faster uniformly stable (non-private) SCO for the non-smooth case is an interesting problem in itself.

\paragraph{Strongly convex case:}
When the loss functions are strongly convex, the optimal (non-private) excess population loss is of the order of $O(1/n)$ rather than $O(1/\sqrt{n})$. The excess loss due to privacy is known to be $\Omega(d/\eps^2 n^2)$. The best known upper bounds for this problem due to~\cite{bassily2014differentially} are $O(\sqrt{d}/\eps n)$. We show a nearly linear time algorithm that has excess loss matching the known lower bounds. As in the convex case, when $d \leq n$, privacy has virtually no additional cost in terms of utility or efficiency. We describe several approaches that achieve these bounds (up to, possibly, a logarithmic overhead). The first approach is based on a folklore reduction to the convex case which can then be used with any of our algorithms for the (non-strongly-convex) convex case. We also give two direct algorithms that rely on a new analysis of SGD with fixed step-size in the strongly convex case. The first algorithm uses iterative localization approach and the second one relies on privacy amplification by iteration.

\section{Preliminaries}
\label{sec:prelims}

\subsection{Convex Loss Minimization}
\label{sec:convOpt}
Let $\cX$ be the domain of data sets, and $\cP$ be a distribution over $\cX$. Let $S=(x_1,\dots,x_n)$ be a dataset drawn i.i.d.\ from~$\cP$. Let $\K\subseteq\Re^d$ be a convex set denoting the space of all models. Let $f\colon \K\times \cX\to\Re$ be a loss function, which is convex in its first parameter (the second parameter is a data point and dependence on this parameter can be arbitrary).
The excess population loss of solution $w$ is defined as
\[
\bE_{x\sim\cP}\left[f(w,x)\right]-\min_{v\in\K}\bE_{x\sim\cP}\left[f(v,x)\right].
\]
In order to argue differential privacy we place certain assumptions on the loss function. To that end, we need the following two definitions of Lipschitz continuity and smoothness.

\begin{defn}[$L$-Lipschitz continuity]
	A function $f\colon\K\to\Re$ is \emph{$L$-Lipschitz continuous} over the domain $\K\subseteq \Re^d$ if the following holds for all $w, w' \in\K$: $\left|f(w)-f(w)\right|\leq L\left\|w-w'\right\|_2$.
	\label{def:lipCont}
\end{defn}

\begin{defn}[$\beta$-smoothness]
	A function $f\colon\K\to\Re$ is \emph{$\beta$-smooth} over the domain $\K\subseteq \Re^d$ if for all $w, w' \in\K$, $\| \nabla f(w) - \nabla f(w')\|_2 \leq \beta \left\|w-w'\right\|_2$.
	\label{def:smoothness}
\end{defn}

\subsection{Probability Measures}
\label{sec:measure}

In this work, we will primarily be interested in the $d$-dimensional Euclidean space $\Re^d$ endowed with the $\ell_2$ metric and the Lebesgue measure.
We say a distribution $\mu$ is \emph{absolutely continuous} with respect to $\nu$ if $\mu(A) = 0$ whenever $\nu(A) = 0$ for all measurable sets $A$. We will denote this by $\mu \ll \nu$.

Given two distributions $\mu$ and $\nu$ on a Banach space $(\cZ,\|\cdot\|)$, one can define several notions of distance between them. The primary notion of distance we consider is \Renyi divergence:
\begin{defn}[\Renyi Divergence~\cite{Renyi61}]\label{def:renyi}
Let $1<\alpha<\infty$ and $\mu, \nu$ be measures with $\mu \ll \nu$. The \Renyi divergence of order $\alpha$ between $\mu$ and $\nu$ is defined as
\[
\Dalpha{\mu}{\nu} \doteq \frac{1}{\alpha-1} \ln \int \left(\frac{\mu(z)}{\nu(z)}\right)^{\alpha} \nu(z)\, \ud z.
\]
Here we follow the convention that $\frac 0 0 = 0$. If $\mu \not\ll \nu$, we define the \Renyi divergence to be $\infty$. \Renyi divergence of orders $\alpha=1,\infty$ is defined by continuity.
\end{defn}

\subsection{(\texorpdfstring{\Renyi}{Renyi}) Differential Privacy}
\label{sec:privacy_prelims}

The notion of differential privacy is by now a de facto standard for statistical data privacy~\cite{DMNS06,Dwork06,DR14-book}.  %
\begin{defn}[\cite{DMNS06,DKMMN06}]\label{def:diffPrivacy}
	A randomized  algorithm $\A$ is \ed-differentially private (\ed-DP) if, for all datasets $S$ and $S'$ that differ in a single data element and for all events $\cO$ in the output space of $\A$, we have
	\[
	\Pr[\A(S)\in \cO] \leq e^{\eps} \Pr[\A(S')\in \cO] +\delta.
	\]
\end{defn}

Starting with Concentrated Differential Privacy~\cite{DworkRothblum-CDP}, definitions that allow more fine-grained control of the privacy loss random variable have proven useful. The notions of zCDP~\cite{BS16-zCDP}, Moments Accountant~\cite{DLDP}, and \Renyi differential privacy (RDP)~\cite{mironov2017renyi} capture versions of this definition. This approach improves on traditional $(\eps, \delta)$-DP accounting in numerous settings, often leading to significantly tighter privacy bounds as well as being applicable when the traditional approach fails~\cite{PATE, PATE2}. %
\begin{defn}[\cite{mironov2017renyi}]\label{def:rDiffPrivacy}
	For $1\leq \alpha\leq \infty$ and $\eps \geq 0$,  a randomized  algorithm $\A$ is \emph{$(\alpha,\eps)$-\Renyi differentially private}, or $(\alpha,\eps)$-RDP if for all neighboring data sets $S$ and $S'$ we have
	\[
	\Dalpha{\A(S)}{\A(S')}\leq \eps.
	\]
\end{defn}

The following two lemmas allow translating \Renyi differential privacy to \ed-differential privacy, and give a composition rule for RDP.
\begin{lem}[\cite{mironov2017renyi,BS16-zCDP}]
	\label{lem:rdp_to_dp}
	If $\A$ satisfies $(\alpha,\eps)$-\Renyi differential privacy, then for all $\delta \in (0,1)$ it also satisfies $\big(\eps+\frac{\ln(1/\delta)}{\alpha-1},\delta\big)$-DP. In particular, if $\A$ satisfies $(\alpha, \alpha \rho^2/2)$-RDP for every $\alpha \geq 1$ then for all $\delta \in (0,1)$ it also satisfies $(\rho^2/2 + \rho \sqrt{2\ln(1/\delta)},\delta)$-DP.
\end{lem}

The standard composition rule for \Renyi differential privacy, when the  outputs
of all algorithms are revealed, takes the following form.
\begin{lem}[\cite{mironov2017renyi}] 
If $\A_1,\dots,\A_k$ are randomized algorithms satisfying, respectively,
$(\alpha,\eps_1)$-RDP,\dots,$(\alpha,\eps_k)$-RDP, then their composition
defined as $(\A_1(S),\dots,\A_k(S))$ is $(\alpha,\eps_1+\dots+\eps_k)$-RDP.
Moreover, the $i$'th algorithm can be chosen on the basis of the outputs of
$\A_1,\dots,\A_{i-1}$.
\end{lem}

\subsection{Contractive Noisy Iteration}

We start by recalling the definition of a contraction.
\begin{defn}[Contraction]
For a Banach space $(\cZ,\|\cdot\|)$, a function $\eff \colon \cZ \to \cZ$ is said to be {\em contractive} if it is 1-Lipschitz. Namely,
for all $x, y \in \cZ$,
\[
\|\eff(x) - \eff(y)\| \leq \|x - y\|.
\]
\end{defn}

A canonical example of a contraction is projection onto a convex set in the Euclidean space.
\begin{proposition}
\label{prop:proj}
Let $\K$ be a convex set in $\Re^d$. Consider the {\em projection operator}:
\[
\Pi_{\K}(x) \doteq \arg \min_{y \in \K} \|x - y\|.
\]
The map $\Pi_{\K}$ is a contraction.
\end{proposition}
Another example of a contraction, which will be important in our work, is a gradient descent step for a smooth convex function. The following is a standard result in convex optimization~\cite{nesterov-book}.
\edef\prop-smooth-contract{\the\value{lem}}
\begin{proposition}\label{prop:smooth-contract}Suppose that a function $f\colon \Re^d \to \Re$ is convex and $\beta$-smooth. Then the function $\eff$ defined~as:
	\[
	\eff(w) \doteq w - \eta \nabla_w f(w)
	\]
	is contractive as long as $\eta \leq 2/\beta$.
\end{proposition}

We will be interested in a class of iterative stochastic processes where we alternate between adding noise and applying some contractive map.
\begin{defn}[Contractive Noisy Iteration (\CNI)]Given an initial random state $X_0 \in \cZ$, a sequence of contractive functions $\eff_t\colon \cZ \to \cZ$, and a sequence of noise distributions $\{\dist_t\}$, we define the Contractive Noisy Iteration (\CNI) by the following update rule:
\[
X_{t+1} \doteq \eff_{t+1}(X_t) + Z_{t+1},
\]
where $Z_{t+1}$ is drawn independently from $\dist_{t+1}$.
For brevity, we will denote the random variable output by this process after $T$ steps as $\CNI_T(X_0, \efs, \dists)$.
\end{defn}

As usual, we denote by $\mu \conv \nu$ the convolution of $\mu$ and $\nu$, that is the distribution of the sum $X+Y$ where we draw $X \sim \mu$ and $Y \sim \nu$ independently.
\begin{defn}
\label{def:renyi-max}
For a noise distribution $\dist$ over a Banach space $(\cZ,\|\cdot\|)$ we measure the magnitude of noise by considering the function that for $a >0$, measures the largest \Renyi divergence of order $\alpha$ between $\dist$ and the same distribution $\dist$ shifted by a vector of length at most $a$:
$$R_\alpha(\dist,a) \doteq \sup_{x\colon \|x\| \leq a} \Dalpha{\dist \conv \mathbf{x}}{\dist} .$$
\end{defn}

We denote the standard Gaussian distribution over $\Re^d$ with variance $\sigma^2$ by $\normal{0}{\sigma^2\Id_d}$. By the well-known properties of Gaussians, for any $x\in \Re^d$, and $\sigma$, $\Dalpha{\normal{0}{\sigma^2\Id_d}}{\normal{x}{\sigma^2 \Id_d}} = \alpha \|x\|_2^2 / 2\sigma^2$. This implies that in the Euclidean space, $R_\alpha(\normal{0}{\sigma^2\Id_d},a) = \frac{\alpha a^2}{2\sigma^2}$.

When $U$ and $V$ are sampled from $\mu$ and $\nu$ respectively, we will often abuse notation and write $\Dalpha{U}{V}$.

\subsection{Privacy Amplification by Iteration}
\label{ss:main-result}
The main result in \cite{FeldmanMTT18} states that
\begin{thm}
\label{thm:pai-general}
Let $X_T$ and $X'_T$ denote the output of $\CNI_T(X_0, \efs, \dists)$ and $\CNI_T(X_0, \efps, \allowbreak \dists)$. Let $s_t \doteq \sup_{x} \|\eff_t(x) - \eff'_t(x)\|$. Let $a_1,\ldots, a_T$ be a sequence of reals and let $z_t \doteq \sum_{i \leq t} s_i - \sum_{i \leq t} a_i$. If $z_t \geq 0$ for all $t$, then
\[
\Dalpha{X_T}{X'_T} \leq \sum_{t=1}^{T} R_\alpha(\dist_t,a_t).
\]
\end{thm}

We now give a simple corollary of this general theorem for the case when the iterative processes differ in a single index and, in addition, the noise distribution with parameter $\sigma$ ensures that \Renyi divergence for a shift of $a$ scales as $a^2/\sigma^2$. As discussed above, this is exactly the case for Gaussian distribution.
\begin{cor}
\label{cor:cni-quadratic}
Let $X_T$ and $X'_T$ denote the output of $\CNI_T(X_0, \efs, \dists)$ and $\CNI_T(X_0, \efps, \allowbreak \dists)$. Let $s_i \doteq \sup_{x} \|\eff_i(x) - \eff'_i(x)\|$. Assume that there exists $t\in[T]$ such that for all $i\neq t$, $s_i = 0$. For $\alpha \geq 1$ assume that there exists $\notbeta$ such that for every $\zeta>0$ and $a\geq 0$, and $i \in [T]$, $R_\alpha(\dist_i,a) \leq \notbeta \frac{a^2}{\sigma_i^2}$ for some $\sigma_i$. Then
\[
\Dalpha{X_T}{X'_T} \leq  \notbeta \frac{s_t^2}{\sum_{i=t}^{T} \sigma_i^2}.
\]
\end{cor}
\begin{proof}
We use Theorem \ref{thm:pai-general} with $a_i=0$ for $i<t$ and $a_i = \frac{s_t \sigma_i^2}{\sum_{v=t}^{T} \sigma_v^2}$.
The resulting bound we get

\[
\Dalpha{X_T}{X'_T} \leq \sum_{i=t}^{T} \notbeta \frac{a_i^2}{\sigma_i^2} = \sum_{i=t}^{T} \notbeta \left(\frac{s_t \sigma_i^2}{\sum_{i=t}^{T} \sigma_i^2} \right)^2 \cdot \frac{1}{\sigma_i^2} = \notbeta s_t^2 \sum_{i=t}^{T} \frac{\sigma_i^2}{\left(\sum_{i=t}^{T} \sigma_i^2 \right)^2} = \notbeta \frac{s_t^2}{\sum_{i=t}^{T} \sigma_i^2}
.\qedhere
\]
\end{proof}
\section{DP SCO via Privacy Amplification by Iteration}
\label{sec:batches}
We start by describing a general version of noisy SGD and analyze its privacy using the privacy amplification by iteration technique from \cite{FeldmanMTT18}.
Recall that in our problem we are given a family of convex loss functions over some convex set $\cK \subseteq \Re^d$ parameterized by $x \in \cX$, that is $f(w,x)$ is convex and differentiable in the first parameter for every $x \in \cX$. Given a dataset $S=(x_1,\ldots,x_n)$, starting point $w_0$, a number of steps $T$, batch size parameters $B_1,\ldots,B_T$ such that $B_t$ are positive integers and $\sum_{t\in[T]} B_t=n$, step sizes $\eta_1,\ldots,\eta_T$, and noise scales $\sigma_1,\ldots,\sigma_T$ the algorithm works as follows. Starting from $w_0 \in \K$ perform the following update $v_{t+1}\doteq w_t - \eta_{t+1} (\nabla_w F_{t+1}(w_t) + \xi_{t+1})$ and $w_{t+1}\doteq\Pi_\K(v_{t+1})$, where $(1)$ $F_{t+1}$ is the average of loss functions for samples in batch $t+1$, that is
$$ F_{t+1}(w) \doteq \frac{1}{B_{t+1}} \sum_{i =1+\sum_{s \leq t} B_s}^{i = \sum_{s \leq t+1} B_s} f(w,x_i) ;$$
$(2)$ $\xi_{t+1}$ is a freshly drawn sample from $\N(0,\sigma_{t+1}^2\Id_d)$; and $(3)$, $\Pi_\K$ denotes the Euclidean projection to set $\K$. We refer to this algorithm as PNSGD$(S,w_0,\{B_t\}, \{\eta_t\},\{\sigma_t\})$ and describe it formally in \cref{alg:ogd}. For a value $a$ we denote the fixed sequence of parameters $(a,\ldots,a)$ of length $T$ by $\{a\}$.
\begin{algorithm}[htb]
	\caption{Projected noisy stochastic gradient descent (PNSGD)}
	\begin{algorithmic}[1]
		\REQUIRE Data set $S=\{x_1,\ldots,x_n\}$, $f\colon \K\times\cX\to\Re$ a convex function in the first parameter, starting point $w_0\in\K$, batch sizes $\{B_t\}$, step sizes $\{\eta_t\}$, noise parameters $\{\sigma_t\}$.
		\FOR{$t\in\{0,\dots,n-1\}$}
			\STATE $v_{t+1}\leftarrow w_t-\eta_{t+1}(\nabla_w F_{t+1}(w_t))+\xi_{t+1})$, where $\xi_{t+1} \sim \cN(0,\sigma_{t+1}^2\Id_d)$.
            \STATE $w_{t+1}\leftarrow \Pi_{\K}\left(v_{t+1}\right)$, where $\Pi_{\K}(w)=\argmin_{\theta\in\K}\|\theta-w\|_2$ is the $\ell_2$-projection on $\K$.
		\ENDFOR
		 \RETURN the final iterate $w_n$.
	\end{algorithmic}
	\label{alg:ogd}
\end{algorithm}

\subsection{Privacy Guarantees for Noisy SGD}
\label{sec:privacy}
As in \cite{FeldmanMTT18}, the key property that allows us to treat noisy gradient descent as a contractive noisy iteration is the fact that for any convex function, a gradient step is contractive as long as the function satisfies a relatively mild smoothness condition (see \cref{prop:smooth-contract}). In addition, as is well known, for any convex set $\cK \in \Re^d$, the (Euclidean) projection to $\cK$ is contractive (see \cref{prop:proj}). Naturally, a composition of two contractive maps is a contractive map and therefore we can conclude that PNSGD$(S,w_0,\{B_t\},\{\eta_t\},\{\sigma_t\})$ is an instance of contractive noisy iteration. More formally, consider the sequence $v_0=w_0,v_1,\ldots,v_n$. In this sequence, $v_{t+1}$ is obtained from $v_t$ by first applying a contractive map that consists of projection to $\K$ followed by the gradient step at $w_t$ and then addition of Gaussian noise of scale $\eta_{t+1} \cdot \sigma_{t+1}$. Note that the final output of the algorithm is $w_n=\Pi_\K(v_n)$ but it does not affect our analysis of privacy guarantees as it can be seen as an additional post-processing step.

More formally, for this algorithm we prove the following privacy guarantees.
\begin{thm}
\label{thm:general-privacy}
Let $\K \subseteq \Re^d$ be a convex set and $\{f(\cdot,x)\}_{x\in \cX}$ be a family of convex $L$-Lipschitz and $\beta$-smooth functions over $\K$. Then, for every batch-size sequence $\{B_t\}_{t\in[T]}$, step-size sequence $\{\eta_t\}_{t\in[T]}$ such that $\eta_t \leq 2/\beta$ for all $t\in [T]$, noise parameters $\{\sigma_t\}_{t\in[T]}$, $\alpha \geq 1$, starting point $w_0\in \K$, and $S\in \cX^n$, PNSGD$(S,w_0,\{B_t\},\{\eta_t\},\{\sigma_t\})$ satisfies $ \left(\alpha, \alpha \cdot \rho^2/2 \right)$-RDP, where $$\rho = 2 L \cdot \max_{t\in [T]} \left\{\frac{\eta_t}{B_t \sqrt{\sum_{s=t}^{T} \eta_s^2 \sigma_s^2} }\right\}.$$
\end{thm}
\begin{proof}
For $k \in [n]$, let $S\doteq (x_1,\ldots,x_n)$ and $S'\doteq (x_1,\ldots,x_{k-1},x'_k,x_{k+1},\ldots,x_n)$ be two arbitrary datasets that differ at index $k$ and let $t$ be the index of the batch in which $k$-th example is used by PNSGD with batch-size sequence $\{B_t\}_{t\in[T]}$.
Note that each $F_{t}$ is an average of $\beta$-smooth, $L$-Lipschitz convex functions and thus is itself $\beta$-smooth, $L$-Lipschitz and convex over $\K$. Thus, as discussed above, under the condition $\eta_t \leq 2/\beta$, the steps of PNSGD$(S,w_0,\{B_t\},\{\eta_t\},\{\sigma_t\})$ are a contractive noisy iteration. Specifically, on the dataset $S$, the \CNI is defined by the initial point~$w_0$, sequence of functions $g_t(w) \doteq \Pi_\K(w) - \eta_t \nabla F_t(\Pi_\K(w))$ and sequence of noise distributions $\dist_t = \N(0,(\eta_t \sigma_t)^2\Id_d)$. Similarly, on the dataset $S'$, the $\CNI$ is defined in the same way with the exception of $g'_t(w) \doteq \Pi_\K(w) - \eta_t \nabla F'_t(\Pi_\K(w))$, where $F'_t$ includes loss function for $x'_k$ instead of $x_k$. Namely, $F'_t(w) = F'_t(w) + (f(w,x'_k) - f(w,x_k))/B_t$.

By our assumption, $f(w,x)$ is $L$-Lipschitz for every $x\in \cX$ and $w\in \K$ and therefore
  \[
\sup_{w} \|g_t(w) - g'_t(w)\|_2 = \frac{\eta_t}{B_t} \sup_{w} \| \nabla f(\Pi_\K(w), x_k) - \nabla f(\Pi_\K(w), x'_k) \|_2 \leq \frac{2\eta_t L}{B_t} .
\]

We can now apply Corollary \ref{cor:cni-quadratic} with $\notbeta = \alpha/2$. Note that $s_t\leq \frac{2\eta_t L}{B_t}$ and  thus we obtain that
\[
\Dalpha{X_n}{X'_n} \leq \frac{\alpha}{2} \cdot \frac{4 L^2 \eta_t^2}{B_t^2} \cdot \frac{1}{\sum_{s=t}^{T} \eta_s^2 \sigma_s^2 } = \frac{2 \alpha L^2 \eta_t^2}{B_t^2 \cdot \sum_{s=t}^{T} \eta_s^2 \sigma_s^2 }.
\]
Maximizing this expression over all indices $i \in [n]$ gives the claim.
\end{proof}

The important property of this analysis is that it allows for batch size to be used to improve the privacy guarantees. The specific batch size choice depends on the step sizes and noise rates. Next we describe the setting of these parameters that ensures convergence at the optimal rate.
\subsection{Utility Guarantees for the Last Iterate of SGD}
\label{subsec:optimization}
In order to analyze the performance of the noisy projected gradient descent algorithm we will use the convergence guarantees for the last iterate of SGD given in \cite{shamir2013stochastic,JainNN19}.
For the purpose of these results we let $F(w)$ be an arbitrary convex function over $\K$ for which we are given an unbiased stochastic (sub-)gradient oracle $G$. That is for every $w\in \K$, $\E[G(w)] \in \partial F(w)$. Let PSGD$(G,w_0,\{\eta_t\}_{t\in [T]})$ denote the execution of the following process: starting from point $w_0$, use the update $w_{t+1} \doteq \Pi_\K(w_t + \eta_{t+1} G(w_t))$ for $t=0,\ldots,T-1$.
Shamir and Zhang \cite{shamir2013stochastic} prove that the suboptimality of the last iterate of SGD with the step size $\eta_t$ being proportional to $1/\sqrt{t}$ scales as $(\log T)/\sqrt{T}$. This variant of SGD relies on relatively large step sizes in the early iterates which would translate into a relatively strong assumption on smoothness in Theorem \ref{thm:general-privacy}. However, it is known \cite{Harvey:19pc} that the analysis in \cite{shamir2013stochastic} also applies to the fixed step size $\eta_t$ scaling as $1/\sqrt{T}$ (in fact, it is simpler and gives a slightly better constants in this case).

\begin{thm}[\cite{shamir2013stochastic}]
\label{thm:sco-last-sz}
Let $\K \subseteq \Re^d$ be a convex body of diameter $D$, let $F(w)$ be an arbitrary convex function over $\K$ and let $G$ be an unbiased stochastic (sub-)gradient oracle $G$ for $F$. Assume that for every $w\in \K$, $\E[\|G(w)\|_2^2]\leq L_G^2$. For $T\in \Nat$ and $w_0 \in \K$, let $w_1,\ldots,w_T$ denote the iterates produced by PSGD$(G,w_0,\{D/(L_G \sqrt{T})\})$. Then
\[
\E[F(w_T)] \leq F^* + \frac{DL_G (2+\ln T)}{\sqrt{T}},
\]
where $F^* \doteq \min_{w\in \K} F(w)$ and the expectation is taken over the randomness of $G$.
\end{thm}

\newcommand{\etajnn}[1]{\bar\eta_{\mathrm{JNN}}(#1)}
\newcommand{\Bjnn}[1]{\bar B_{\mathrm{JNN}}(#1)}

Further, Jain et al.~\cite{JainNN19} show that the $\ln T$ factor can be eliminated by using faster decaying rates. Their step-size schedule is defined as follows.
\begin{defn}
\label{def:jnn-rates}
For an integer $T$, let $\ell = \lceil \log_2{T} \rceil$. For $0 \leq i \leq \ell$, let $T_i = T- \lceil T \cdot 2^{-i} \rceil$ and let $T_{\ell+1} = T$. For a constant $c$, every $0 \leq i \leq \ell$ and $T_i < t \leq T_{i+1}$, we define $\eta_t = \frac{c 2^{-i}}{\sqrt{T}}$. We denote the resulting sequence of step sizes by $\etajnn{c}$.
\end{defn}
Jain et al.~\cite{JainNN19} prove that the following guarantees hold for SGD with step sizes given by $\etajnn{c}$.
\begin{thm}[\cite{JainNN19}]
\label{thm:sco-last-jnn}
Let $\K \subseteq \Re^d$ be a convex body of diameter $D$, let $F(w)$ be an arbitrary convex function over $\K$ and let $G$ be an unbiased stochastic (sub-)gradient oracle $G$ for $F$. Assume that for every $w\in \K$, $\E[\|G(w)\|_2^2]\leq L_G^2$. For $T\in \Nat$ and $w_0 \in \K$, let $w_1,\ldots,w_T$ denote the iterates produced by PSGD$(G,w_0,\etajnn{D/L_G})$. Then
\[
\E[F(w_T)] \leq F^* + \frac{15 DL_G}{\sqrt{T}},
\]
where $F^* \doteq \min_{w\in \K} F(w)$ and the expectation is taken over the randomness of $G$.
\end{thm}
We remark that the results in \cite{JainNN19} are stated for an oracle $G$ that gives (sub)-gradients bounded by $G_L$ almost surely. This condition is necessary for the high-probability version of their result but a bound on the variance of $G$ suffices to upper bound $\E[F(w_T)]$. In addition, while the results are stated for a fixed gradient oracle, the same results hold when a different stochastic gradient oracle $G_t$ is used in step $t$ as long as all the oracles satisfy the assumptions (namely, $\E[G_t(w)] \in \partial F(w)$ and $\E[\|G_t(w)\|_2^2]\leq L_G^2$ for all $t$).

\subsection{Snowball-SGD}
Finally we derive the privacy and utility guarantees for noisy SGD by calculating the batch sizes needed to ensure the privacy guarantees for the settings in Theorems \ref{thm:sco-last-sz} and \ref{thm:sco-last-jnn}. The sum of batch sizes in turn gives us the number of samples $n$ necessary to implement $T$ steps of these algorithms. The resulting batch sizes will be proportional to $\sqrt{d/(T-t+1)}$ and we refer to such batch size schedule as Snowball-SGD.

\begin{thm}
\label{thm:sz-privacy-utility}
Let $\K \subseteq \Re^d$ be a convex set of diameter $D$ and $\{f(\cdot,x)\}_{x\in \cX}$ be a family of convex $L$-Lipschitz and $\beta$-smooth functions over $\K$. For $T \in \Nat$, $\rho > 0$, and all $t\in [T]$ let $B_t = \lceil 2\sqrt{d/(T-t+1)}/\rho \rceil$, $n = \sum_{t\in [T]} B_t$, $\eta = D/(L \sqrt{2T})$, $\sigma = L/\sqrt{d}$,
If $\eta \leq 2/\beta$ then for all $\alpha \geq 1$, starting point $w_0\in \K$, and $S\in \cX^n$, PNSGD$(S,w_0,\{B_t\},\{\eta\},\{\sigma\})$ satisfies $ \left(\alpha, \alpha \cdot \rho^2/2 \right)$-RDP. Further, if $S$ consists of samples drawn i.i.d.~from a distribution $\cP$, then $n \leq T + 4\sqrt{d T}/\rho$ and
\[
\E[F(w_T)] \leq F^* + \frac{\sqrt{8} D L (2+\ln T)}{\sqrt{T}} \leq F^* + \sqrt{32} D L \cdot \ln(10n) \cdot \left(\frac{1}{\sqrt{n}} + \frac{2\sqrt{d}}{\rho n} \right),
\]
where, for all $w\in \K$, $F(w) \doteq \E_{x \sim \cP}[f(w,x)]$, $F^* \doteq \min_{w\in \K} F(w)$ and the expectation is taken over the random choice of $S$ and noise added by PNSGD.
\end{thm}
\begin{proof}
We first establish the privacy guarantees. By Theorem \ref{thm:general-privacy}, all we need is to verify that for our choice of $\{B_t\},\sigma$ and $\eta$ we have for every $t\in [T]$,
$$4 L^2 \cdot \max_{t\in [T]} \left\{\frac{\eta^2}{B_t^2 \cdot \sum_{s=t}^{T} \eta^2 \sigma^2 }\right\} = 2 L^2 \max_{t\in [T]} \left\{\frac{1}{B_t^2 \cdot (T-t+1) L^2/d }\right\}
 \leq \rho^2.$$
This implies that $$n = \sum_{t\in [T]} \left\lceil \sqrt{\frac{4d}{\rho^2(T-t+1)}} \right\rceil \leq \sum_{t\in [T]} \sqrt{\frac{4d}{\rho^2(T-t+1)}}+1 = T + \frac{2\sqrt{d}}{\rho} \sum_{t\in [T]} \frac{1}{\sqrt{t}} \leq T + \frac{4\sqrt{d T}}{\rho} ,$$
where we used the fact that $\sum_{t\in [T]} \frac{1}{\sqrt{t}} \leq 2(\sqrt{T+1}-1) +1 \leq 2\sqrt{T}$.

To establish the utility guarantees, we first note that for all $t\in [T]$, $$\nabla F_t(w) = \frac{1}{B_t} \sum_{i =1+\sum_{s \leq t-1} B_s}^{i = \sum_{s \leq t} B_s} \nabla_w f(w,x_i) .$$ Thus for $S$ sampled i.i.d.~from $\cP$ and index $i$ in batch $t$, $\E[\nabla_w f(w,x_i)] = \nabla F(w)$. In particular, for $\xi_t\sim \N(0,\sigma^2)$, $\E[\nabla F_t(w) + \xi_t] = \nabla F(w)$ and therefore each $\nabla F_t(w) + \xi_t$ gives an independent sample from a stochastic gradient oracle for $F$. Our setting of the noise scale $\sigma = L/\sqrt{d}$ ensures that for every $t \in [T]$
$$\E_{S\sim \cP^n, \xi_t\sim \N(0,\sigma^2)}\left[\left\|\nabla F_t(w) + \xi_t\right\|_2^2\right]= \frac{\E_{x\sim \cP}\left[\left\|\nabla_w f(w,x)\right\|_2^2\right]}{B_t} + d \sigma^2 \leq \frac{L^2}{B_t} + L^2 \leq 2 L^2 .$$
This implies that for our choice of parameters PNSGD$(S,w_0,\{B_t\},\{\eta\},\{\sigma\})$ can be seen as an execution PSGD$(G,w_0,\{D/(L_G \sqrt{T})\})$ with stochastic gradient oracles with variance upper-bounded by $L_G^2= 2 L^2$. Plugging this value in Theorem \ref{thm:sco-last-sz} gives our bound on the utility of the algorithm.
To obtain the bound in terms of $n$ we note that
$n \leq  T + \frac{4\sqrt{d T}}{\rho}$, implies that $T \geq  \frac{n^2}{16d/\rho^2 + 4n}$ and thus
\begin{align*}
    \frac{1}{\sqrt{T}}
    \leq
    \frac{2}{\sqrt{n}} + \frac{4\sqrt{d}}{\rho n}
    .&\qedhere
\end{align*}
\end{proof}

Next, we give a differentially private version of the step-size schedule from \cite{JainNN19}. %

\begin{thm}
\label{thm:jnn-privacy-utility}
Let $\K \subseteq \Re^d$ be a convex set of diameter $D$ and $\{f(\cdot,x)\}_{x\in \cX}$ be a family of convex $L$-Lipschitz and $\beta$-smooth functions over $\K$. For $T \in \Nat$, $\rho > 0$, and all $t\in [T]$ let $B_t = \lceil 4\sqrt{3d/(T-t+1)}/\rho \rceil$, $n = \sum_{t\in [T]} B_t$, $\{\eta_t\} = \etajnn{D/(\sqrt{2}L)}$, $\sigma = L/\sqrt{d}$,
If $\eta_1 \leq 2/\beta$ then for all $\alpha \geq 1$, starting point $w_0\in \K$, and $S\in \cX^n$, PNSGD$(S,w_0,\{B_t\},\{\eta_t\},\{\sigma\})$ satisfies $\left(\alpha, \alpha \cdot \rho^2/2 \right)$-RDP. Further, if $S$ consists of samples drawn i.i.d.~from a distribution $\cP$, then $n \leq T + 4\sqrt{d T}/\rho$ and
\[
\E[F(w_T)] \leq F^* + \frac{15 \sqrt{2} D L}{\sqrt{T}} \leq 30 \sqrt{2} D L \cdot \left(\frac{1}{\sqrt{n}} + \frac{4\sqrt{3d}}{\rho n} \right),
\]
where, for all $w\in \K$, $F(w) \doteq \E_{x \sim \cP}[f(w,x)]$, $F^* \doteq \min_{w\in \K} F(w)$ and the expectation is taken over the random choice of $S$ and noise added by PNSGD.
\end{thm}
\begin{proof}
The utility guarantees for this algorithm follow from the same argument as in the proof of Theorem \ref{thm:sz-privacy-utility} together with Theorem \ref{thm:sco-last-jnn}. As before, by Theorem \ref{thm:general-privacy}, all we need to establish the privacy guarantees is to verify that for our choice of $\{B_t\},\sigma$ and $\{\eta_t\}$ we have for every $t\in [T]$,
\equ{
\label{eq:jnn-privacy-cond}
4 L^2 \cdot \max_{t\in [T]} \left\{\frac{\eta_t^2}{B_t^2 \cdot \sum_{s=t}^{T} \eta_s^2 \sigma^2 }\right\}
 \leq \rho^2.
}
We first observe that for $t=T$ we have that
\equ{\label{eq:case-T}
\frac{\eta_t^2}{B_t^2 \cdot \sum_{s=t}^{T} \eta_s^2 \sigma^2} = \frac{d}{B_T^2 L^2} \leq \frac{\rho^2}{48 L^2}.}
For $t \in [T-1]$, let $i$ be such that $T_i<t\leq T_{i+1}$.
Then we note that for $c= D/(\sqrt{2}L)$
$$\eta_t = c \frac{2^{-i}}{\sqrt{T}} \geq  c \frac{\lceil T 2^{-i} \rceil - 1 }{T \sqrt{T}} \geq c \frac{T-T_i-1}{T^{3/2}} \geq c \frac{T-t}{T^{3/2}}.$$
and therefore
 $$\sum_{s=t}^{T} \eta_s^2 \geq c^2 \sum_{s=t}^{T} \frac{(T-s)^2}{T^3} \geq \frac{c^2}{T^3} \cdot \frac{(T-t)^2(T-t+1)}{3} .$$
In addition, using the fact that for $t\leq T-1$, $i\leq \ell -1$ we have that
$$\eta_t  = 2 c \frac{2^{-i-1}}{\sqrt{T}} \leq 2 c \frac{\lceil T 2^{-i-1} \rceil }{T \sqrt{T}} = 2 c \frac{T-T_{i+1}}{T^{3/2}} \leq 2c \frac{T-t}{T^{3/2}} .$$
Thus
$$\frac{\eta_t^2}{B_t^2 \cdot \sum_{s=t}^{T} \eta_s^2 \sigma^2} =  \frac{1}{B_t^2 \sigma^2}\cdot \frac{4c^2  (T-t)^2}{T^3} \cdot \frac{3T^3}{c^2 (T-t)^2(T-t+1)} = \frac{12d}{B_t^2 L^2 (T-t+1)} \leq \frac{\rho^2}{4L^2}.$$
Plugging this and eq.~\eqref{eq:case-T} into eq.\eqref{eq:jnn-privacy-cond} we obtain that the privacy condition holds.

As in the proof of Theorem \ref{thm:jnn-privacy-utility}, we obtain that $$n = \sum_{t\in [T]} B_t  \leq T + \frac{8\sqrt{3}\sqrt{d T}}{\rho} $$
and thus $T \geq  \frac{n^2}{192d/\rho^2 + 4n}$. This means that
$$\frac{1}{\sqrt{T}} \leq \frac{2}{\sqrt{n}} + \frac{8\sqrt{3d}}{\rho n} ,$$
implying the claimed bound on utility in terms of $n$.
\end{proof}

As a corollary we get the proof of our main claim.
\begin{cor}[Thm.~\ref{thm:main-intro} restated]
\label{cor:jnn-privacy-utility}
Let $\K \subseteq \Re^d$ be a convex set of diameter $D$ and $\{f(\cdot,x)\}_{x\in \cX}$ be a family of convex $L$-Lipschitz and $(2D\sqrt{2T}/L)$-smooth functions over $\K$. For every $\rho >0$, there exists an algorithm $\A$ that given
a starting point $w_0\in \K$, and $S\in \cX^n$ returns a point $\hat w$. For all $\alpha \geq 1$, $\A$ satisfies $\left(\alpha, \alpha \cdot \rho^2/2 \right)$-RDP and uses $n$ evaluations of the gradient of $f(w,x)$.  Further, if $S$ consists of samples drawn i.i.d.~from a distribution $\cP$ over $\cX$, then
\[
\E[F(\hat w)] \leq F^* + O\left( D L \cdot \left(\frac{1}{\sqrt{n}} + \frac{\sqrt{d}}{\rho n} \right)\right),
\]
where, for all $w\in \K$, $F(w) \doteq \E_{x \sim \cP}[f(w,x)]$, $F^* \doteq \min_{w\in \K} F(w)$ and the expectation is taken over the random choice of $S$ and randomness of $\A$.
\end{cor}

\newcommand{\lr}[1]{\mathopen{}\left(#1\right)}
\newcommand{\Lr}[1]{\mathopen{}\big(#1\big)}
\newcommand{\LR}[1]{\mathopen{}\Big(#1\Big)}
\newcommand{\lrnorm}[1]{\mathopen{}\left\|#1\right\|}

\newcommand{\set}[1]{\{#1\}}
\newcommand{\lrset}[1]{\mathopen{}\left\{#1\right\}}
\newcommand{\Lrset}[1]{\mathopen{}\big\{#1\big\}}
\newcommand{\LRset}[1]{\mathopen{}\Big\{#1\Big\}}

\newcommand{\ceil}[1]{\lceil #1 \rceil}
\newcommand{\lrceil}[1]{\mathopen{}\left\lceil #1 \right\rceil}
\newcommand{\Lrceil}[1]{\mathopen{}\big\lceil #1 \big\rceil}
\newcommand{\LRceil}[1]{\mathopen{}\Big\lceil #1 \Big\rceil}

\section{Localization-Based Algorithms}
\label{sec:localize}
In this section, we describe the Iterative Localization framework, and give two instantiations of it. Our localization algorithm will be based on adding Gaussian noise to an algorithm whose output has low $L_2$-sensitivity (also referred to as uniform stability of the parameter). We first briefly recall the relevant definitions and the resulting privacy guarantees.
\begin{defn}
A deterministic algorithm (or function) $\cA : \cX^n \to \reals^d$ has $L_2$-sensitivity of $\gamma$ if for all pairs of datasets $S,S'\in \cX^n$ that differ in a single element we have that $\norm{\cA(S)-\cA(S')}_2 \leq \gamma$.
\end{defn}

The well-known property of the Gaussian mechanism is that it can convert any algorithm with bounded $L_2$-sensitivity to a differentially private one.
\begin{lem} \label{lem:gauss-mech}
Let $\cA : \cX^n \to \reals^d$ be a deterministic function with $L_2$-sensitivity $\gamma$. Then for any $\rho > 0$, an algorithm that outputs $\cA(S) + \xi$ where $\xi \sim \mathcal{N}(0,\frac{\gamma^2}{\rho^2} \Id_d)$ satisfies $(\alpha, \tfrac{1}{2}\alpha \rho^2)$-RDP for all $\alpha \geq 1$.
\end{lem}

Suppose we have an algorithm $\cA$ that given a point $w \in B(w^\ast, D)$ and a sequence of samples from $F$, outputs a $w' \in B(w^\ast, D/4)$. We will want this $\cA$ to have small sensitivity, and small suboptimality, both of which scale, say, linearly with $D$. Given such an $\cA$, we can iteratively invoke it with geometrically decreasing $D$, adding noise at the end of each phase to ensure privacy. Crucially, once the diameter bound $D$ becomes small enough, the noise added is small enough that the suboptimality due the added noise is negligible. This would allow us to incur only a logarithmic overhead in terms of sample complexity, while ensuring privacy and good utility bounds.

We next describe two instantiations of this Iterative Localization framework. To get better bounds, we only bound the second moment of the distance $\norm{w'-w^\ast}$, instead of requiring that $\norm{w'-w^\ast}$ is uniformly bounded. The first instantiation uses SGD as algorithm $\cA$, and applies to convex functions. The sensitivity bound here comes from bounding the step sizes and holds under mild smoothness assumptions. The second instantiation will apply to arbitrary convex functions, and optimizes a regularized objective to ensure a sensitivity bound.

\subsection{SGD-Based Iterative Localization}
Our algorithm is based on a sequence of phases such that each phase (implicitly) localizes an approximate minimizer of the population loss. Specifically, given a point $w_i$ such that for some $w_i^\ast$, $\E[\|w_i-w_i^\ast\|_2^2] \leq D$, the algorithm outputs a point $w_{i+1}$ such that for some point $w_{i+1}^\ast$, $\E[\|w_{i+1}-w_{i+1}^\ast\|_2^2] \leq D/4$ and, in addition, $\E[F(w_{i+1}^\ast)] - \E[F(w_i^\ast)] \leq \tau/2^{i}$ where $\tau$ is the desired excess loss.

Our algorithm relies on the fact that SGD on sufficiently smooth loss functions has low $L_2$-sensitivity \cite{hardt2015train,feldman2019high}.
\begin{lem} \label{lem:hrs}
Each iterate of one-pass online projected gradient descent with fixed step size $\eta$ over a sequence of $\beta$-smooth $L$-Lipschitz convex functions has $L_2$-sensitivity of at most $2L\eta$ as long as $\eta \leq 2/\beta$. In particular, the same applies to the average of all the iterates.
\end{lem}

\begin{algorithm}[htb]
	\caption{Phased-SGD algorithm}
	\begin{algorithmic}[1]
		\REQUIRE Data set $S=\{x_1,\ldots,x_n\}$, convex $f\colon \K\times\cX\to\Re$, initial point $w_0\in\K$, step size $\eta$, privacy parameter $\rho$.
		\STATE set $k=\ceil{\log_2 n}$
		\FOR{$i=1,\ldots,k$}
		\STATE set $n_i = 2^{-i} n$ and $\eta_i = 4^{-i} \eta$.
        \STATE initialize an PSGD algorithm (over domain $\K$) at $w_{i-1}$ and run with step size $\eta_i$ for $n_i$ steps; let $\overline{w}_i$ be the average iterate.
        \STATE set $w_i = \overline{w}_i + \xi_i$, where $\xi_i \sim \cN(0,\sigma_i^2 \Id_d)$ with
        $
        	\sigma_i = 4L\eta_i\big/\rho.%
        $
		\ENDFOR
		\RETURN the final iterate $w_k$.
	\end{algorithmic}
	\label{alg:sgd-phased}
\end{algorithm}

\begin{thm} \label{thm:sgd-phased}
Assume that $\norm{w_0-w^\ast}_2 \leq D$ (this is the case, for example, when $\K$ has diameter at most $D$), and set
$$
	\eta
	=
	\frac{D}{L} \min\lrset{ \frac{4}{\sqrt{n}}, \frac{\rho}{\sqrt{d}} }
	.
$$
Then for the output of \cref{alg:sgd-phased}, we have
\begin{align*}
	\E[F(w_k)] - F(w^\ast)
	\leq
	10 LD \lr{ \frac{1}{\sqrt{n}} + \frac{\sqrt{d}}{\rho n} }
	,
\end{align*}
provided that $\eta \leq 2/\beta$.
\end{thm}

To prove the theorem, we first provide utility and privacy guarantees for each individual phase of the algorithm.

\begin{lem} \label{lem:sgd-phase}
Assume that $\eta_i \leq 2/\beta$.
Then for any $\alpha\geq 1$, the output $w_i$ of phase $i$ in \cref{alg:sgd-phased} %
satisfies $(\alpha,\alpha \rho^2/2)$-RDP,
and for any $w \in \K$,
\begin{align}
    \E[F(\overline{w}_i)] - F(w)
    \leq
    \frac{\E[\norm{w_{i-1}-w}_2^2]}{2\eta_i n_i} + \frac{\eta_i L^2}{2}
    .
\end{align}
\end{lem}

\begin{proof}
The privacy guarantee follows from \cref{lem:gauss-mech} together with the fact that PSGD (when viewed as a deterministic mapping from a data set to a final iterate) with step size $\eta_i \leq 2/\beta$ has $L_2$-sensitivity bounded by $2L\eta_i$ (this is a consequence of \cref{lem:hrs}).
The utility guarantee follows from standard convergence bounds for PSGD (e.g., Lemma 7 of \cite{hazan2014beyond}).
\end{proof}
We can now prove \cref{thm:sgd-phased}.

\begin{proof}[Proof of \cref{thm:sgd-phased}]
Denote $\overline{w}_0 = w^\ast$ and $\xi_0 = w_0-w^\ast$; by assumption, $\norm{\xi_0}_2 \leq D$.
Using \cref{lem:sgd-phase}, the total error of the algorithm can be bounded by
\begin{align*}
	\E[F(w_k)] - F(w^\ast)
	&=
	\sum_{i=1}^k \E[F(\overline{w}_i) - F(\overline{w}_{i-1})]
	+ \E[F(w_k) - F(\overline{w}_k)]
	\\
	&\leq
	\sum_{i=1}^k \lr{ \frac{\E[\norm{\xi_{i-1}}_2^2]}{2\eta_i n_i}
	+ \frac{\eta_i L^2}{2} } + L \cdot \E[\norm{\xi_k}_2]
	.
\end{align*}
Recall that by definition $\eta \leq (D/L)\cdot(\rho/\sqrt{d})$, so that for all $i \geq 0$,
\begin{align*}
    \E[\norm{\xi_i}_2^2]
    =
    d\sigma_i^2
    =
    d(4^{-i}  L \eta/\rho)^2
    \leq
    (4^{-i} D)^2
    .
\end{align*}
In particular, we have $\E[\norm{\xi_k}_2] \leq \sqrt{\E[\norm{\xi_k}_2^2]} = 4^{-k}D$.
Hence,
\begin{align*}
	\E[F(w_k)] - F(w^\ast)
	&\leq
	\sum_{i=1}^k 2^{-i} \lr{ \frac{8D^2}{\eta n} + \frac{\eta L^2}{2} }
	+ 4^{-k} LD
	\\
	&\leq
	\sum_{i=1}^\infty 2^{-i} LD \lr{ \frac{8}{n} \max\lrset{\sqrt{n},\frac{\sqrt{d}}{\rho}} + \frac{1}{2\sqrt{n}} }
	+ \frac{LD}{n^2}
	\\
	&\leq
	9LD \lr{ \frac{1}{\sqrt{n}} + \frac{\sqrt{d}}{\rho n} } + \frac{LD}{n^2}
	.&&\qedhere
\end{align*}
\end{proof}

\subsection{Non-Smooth DP-SCO: Phased ERM}

In this section, we demonstrate that the general approach based on localization can also be applied to the non-smooth case. We only require $f(\cdot,x)$ to be convex and $L$-Lipschitz for any $x \in \cX$. Our algorithm is similar to the one in the previous section, except that we replace the PSGD subroutine in step 3 of the algorithm with a regularized ERM computation. The $L_2$ regularization in this case is the standard technique for ensuring low sensitivity that we require. In addition, low-sensitivity ensures uniform stability and thus generalization of the solution to the population. To get a more efficient algorithm, we use an approximate optimizer instead of an exact one. The suboptimality of this optimization should be small enough that the sensitivity of the resulting algorithm can still be controlled. To solve the regularized problem, we employ SGD that ensures the suboptimality bound, and hence the sensitivity bound, with high probability. To allow for a small failure probability of this approach, we will only give $(\eps,\delta)$-DP guarantees for the algorithm.
We will use the following standard variant of \cref{lem:gauss-mech}:
\begin{lem}
\label{lem:gauss-mech-whp}
Let $\cA : \cX^n \to \reals^d$ be a randomized function such that for all pairs of datasets $S,S'\in \cX^n$ that differ in a single element there is a coupling such that  $\norm{\cA(S)-\cA(S')}_2 \leq \gamma$ except with probability $\delta$.  Then for any $\rho > 0$, an algorithm that outputs $\cA(S) + \xi$ where $\xi \sim \mathcal{N}(0,\frac{\gamma^2}{\eps^2} \ln\frac 1 \delta \, \Id_d)$ satisfies $(\eps, 2\delta)$-DP.
\end{lem}

\begin{algorithm}[htb]
	\caption{Phased-ERM algorithm}
	\begin{algorithmic}[1]
		\REQUIRE Data set $S=\{x_1,\ldots,x_n\}$, convex $f\colon \K\times\cX\to\Re$, initial point $w_0\in\K$, step size $\eta$, privacy parameters $\eps, \delta$.
		\STATE set $k=\ceil{\log_2 n}$
		\FOR{$i=1,\ldots,k$}
		\STATE set $n_i = 2^{-i} n$ and $\eta_i = 4^{-i} \eta$.
        \STATE compute $\tilde w_i \in \K$ such that $F_i(\tilde w_i)-\argmin_{w \in \K} F_i(w) \leq L^2\eta_i/n_i$ with prob. $(1-\delta)$
				for
        \begin{align*}
          F_i(w)
					=
					\frac{1}{n_i} \sum_{t=1}^{n_i} f(w,x_t) + \frac{1}{\eta_i n_i} \norm{w-w_{i-1}}_2^2
          .
        \end{align*}
        \STATE set $w_i = \tilde{w}_i + \xi_i$, where $\xi_i \sim \cN(0,\sigma_i^2 \Id_d)$ with
        $
        	\sigma_i = 4L(\eta_i/\eps) \sqrt{\ln(1/\delta)}.
        $
		\ENDFOR
		\RETURN the final iterate $w_k$.
	\end{algorithmic}
	\label{alg:erm-phased}
\end{algorithm}

We first prove the relevant properties of the regularized ERM algorithm.

\begin{lem} \label{lem:erm-phase}
The output $w_i$ of phase $i$ of \cref{alg:erm-phased} satisfies $(\eps, 2\delta)$-DP, and for any $w \in \K$,
\begin{align}
  \E[F(\tilde w_i)] - F(w)
  \leq
  \frac{\E[\norm{w_{i-1}-w}_2^2]}{\eta_i n_i} + 3\eta_i L^2
  .
\end{align}
Further, $\tilde w_i$ can be found using $O(n_i^2 \log(1/\delta))$ gradient computations on $f$.
\end{lem}

\begin{proof}
The objective $F_i$ minimized in phase $i$ is $L$-Lipschitz and $\lambda_i$-strongly convex for $\lambda_i = 2/(\eta_i n_i)$; denote by $\overline{w}_i \in \K$ its minimizer.
From the results in \cite{BousquetElisseeff-2002,SSSS} we know that the minimizer $\overline{w}_i$ has $L_2$ sensitivity bounded by $4L/(\lambda_i n_i) = 2L\eta_i$, and furthermore,
\begin{align*}
  \E[F(\overline{w}_i)] - F(w)
  \leq
  \frac{\lambda_i}{2} \E[\norm{w_{i-1}-w}_2^2] + \frac{4L^2}{\lambda_i n_i}
  =
  \frac{\E[\norm{w_{i-1}-w}_2^2]}{\eta_i n_i} + 2L^2\eta_i
  .
\end{align*}
(This is a slight modification of Theorem 7 in \cite{SSSS}.)
For the approximate minimizer $\tilde w_i$, we have by strong convexity that except with probability $\delta$
\begin{align*}
	\frac{\lambda_i}{2} \norm{\tilde{w}_i - \overline{w}_i}^2
	\leq
	F_i(\tilde{w}_i) - F_i(\overline{w}_i)
	\leq
	\frac{L^2 \eta_i}{n_i}
	,
\end{align*}
which implies that $\norm{\tilde{w}_i - \overline{w}_i} \leq L\eta_i$.
In particular, $\tilde w_i$ has sensitivity of at most $4L\eta_i$,
which gives the privacy guarantee via \cref{lem:gauss-mech-whp}.
Finally, for any $w \in \K$ we have
\begin{align*}
	\E[F(\tilde w_i)] - F(w)
	=
	\E[F(\overline w_i) - F(w)] + \E[F(\tilde w_i) - F(\overline w_i)]
	\leq
	\frac{\E[\norm{w_{i-1}-w}_2^2]}{\eta_i n_i} + 3L^2\eta_i
	,
\end{align*}
which implies the claim on utility.
Finally, to obtain the running time statement, we use the fact that for optimizing an $L$-Lipschitz and $\lambda$-strongly convex function to within $\Delta$ accuracy with probability $\geq 1-\delta$ using SGD, one needs $O((L^2/\lambda\Delta) \log(1/\delta))$ stochastic gradient computations (e.g., \cite{HarveyLPR19}).
Hence, the number of gradient calls needed for computing $\tilde w_i$, being an $\Delta_i$-approximate minimizer of a $\lambda_i$-strongly convex function for $\Delta_i = L^2\eta_i/n_i$ and $\lambda_i = 1/(\eta_i n_i)$, is
$
	O((L^2/\lambda_i \Delta_i) \log(1/\delta))
	=
	O(n_i^2 \log(1/\delta))
	.
$
\end{proof}

The proof of the following result is identical to that \cref{thm:sgd-phased}, with \cref{lem:erm-phase} replacing \cref{lem:sgd-phase}.

\begin{thm} \label{thm:erm-phased}
Assume that $\norm{w_0-w^\ast}_2 \leq D$, and set
$$
	\eta
	=
	\frac{D}{L} \min\lrset{ \frac{4}{\sqrt{n}}, \frac{\eps}{\sqrt{d\ln(1/\delta)}} }
	.
$$
Then for the output of \cref{alg:sgd-phased}, we have
\begin{align*}
	\E[F(w_k)] - F(w^\ast)
	=
  O\lr{LD \lr{ \frac{1}{\sqrt{n}} + \frac{\sqrt{d\ln(1/\delta)}}{\eps n}}}
	.
\end{align*}
Further, a version of this algorithm can be implemented with $O(n^2\sqrt{\ln(1/\delta)})$ stochastic gradient computations.
\end{thm}

\section{The Strongly Convex Case}

Suppose that the population loss of interest $F$ is $\lambda$-strongly convex and $L$-Lipschitz over the domain $\K$.
In this case, the optimal statistical rate is $O(L^2 \big/ \lambda n)$~\citep{hazan2014beyond}, and the private ERM can be optimized with an error of $\wt{O}(dL^2 \big/ \eps^2\lambda n^2)$. The best known bound for Private Stochastic Convex Optimization for this case is due to~\cite{bassily2014differentially} who give an upper bound of $\wt{O}(L^2\sqrt{d} \big/ \lambda \eps n)$. As in the convex case, we show that the optimal rate is in fact the larger of the two lower bounds, and is attained by a linear-time algorithm.

We first show that an asymtotically-optimal algorithm and linear-time for this case can be obtained via a folklore reduction to the convex case (see, e.g., \cite{hazan2014beyond} for a similar instantiation of this reduction). We then give two new algorithms for the strongly convex case: one based on the iterative localization and the other based on privacy amplification by iteration. The algorithms are simpler and require weaker assumption on the condition number than the reduction-based approach. Both of the new algorithms rely on a new analysis of SGD with fixed step-size in the strongly convex case.

\subsection{Reduction to the Convex Case}
Assume a private stochastic (non-strongly) convex optimization algorithm $\cA$ with the following utility guarantee when initialized at $w_0 \in \K$:
\begin{align*}
    \E[F(w_\cA)] - F(w^\ast)
	 \leq
	 cLD \lr{\frac{1}{\sqrt{n}} + \frac{\sqrt{d}}{\rho n}}
\end{align*}
for some universal constant $c \geq 1$, where $D > 0$ is such that $\norm{w_0-w^\ast}_2 \leq D$.
(E.g., this can be one of \cref{alg:ogd,alg:sgd-phased,alg:erm-phased} under their respective assumptions and settings of $\rho$.)
Consider the following algorithm: starting from a given $w_0 \in \K$, repeat the private optimization algorithm $\cA$ for $k=\ceil{\log\log{n}}$ times, where run $i=1,\ldots,k$ is initialized at the output of the previous phase and is run for $n_i = 2^{i-2} n/\log{n}$ iterations.
We prove the following:
\begin{thm} \label{thm:strongly}
The algorithm described above is private (with the same privacy parameters as of $\cA$), and using no more than $n$ samples outputs a solution whose expected population loss is at most
$$
    O\lr{\frac{L^2}{\lambda} \lr{\frac 1{n} + \frac{d}{\rho^2 n^2}}}
    .
$$
\end{thm}

This is the optimal rate under strong convexity assumptions.
Further, under $\beta$-smoothness assumptions and when the condition number is $\beta/\lambda = O(\max\{\sqrt{n/\log{n}}, \sqrt{d}/\rho\})$, the inner stochastic convex optimization problems are sufficiently smooth so that we can use \cref{alg:sgd-phased} as the basic private optimization algorithm (with step sizes $\eta_i$ that satisfy $\beta \leq 1/\eta_i$, as $n_i = \Omega(n/\log{n})$ for all $i$ in the reduction) and get a linear time algorithm for stochastic strongly convex optimization.
Without any smoothness assumptions, we can invoke the reduction with \cref{alg:erm-phased} and get a quadratic-time algorithm with the optimal rate.
In \cref{app:strongly-convex} we show how the constraint on the condition number can be relaxed all the way up to $\beta/\lambda = O(n/\log n)$ via a more careful argument that utilizes our iterative localization framework directly.

\begin{proof}[Proof of \cref{thm:strongly}]
First, observe that the total number of samples used by the algorithm is indeed
$\sum_{i=1}^k n_i \leq 2^{k-1} n/\log{n} \leq n.$
Denote the output of phase $i$ by $w_i$; let $\Delta_i = \E[F(w_i)] - F(w^\ast)$ be the expected suboptimality after phase $i$, and let $D_i^2 = \E[\norm{w_i-w^\ast}^2]$ for all $i$.
The $\lambda$-strong convexity of $F$ implies $\tfrac{1}{2}\lambda D_i^2 \leq \Delta_i$ for all $i \geq 0$.
Thus, by the guarantee of the private convex optimization algorithm, we have for all $i$ that
\begin{align*}
	\Delta_{i+1}
	\leq
	c L D_i \lr{\frac{1}{\sqrt{n_i}} + \frac{\sqrt{d}}{\rho n_i}}
	\leq
     c L \sqrt{\frac{2\Delta_i}{\lambda}} \lr{\frac{1}{\sqrt{n_i}} + \frac{\sqrt{d}}{\rho n_i}}
	.
\end{align*}
Let us denote by $E_i$ the expression $c^2 (2L^2/\lambda) \Lr{\frac{1}{\sqrt{n_i}} + \frac{\sqrt{d}}{\rho n_i}}^2$.
Since $E_i/E_{i+1} \leq 4$ (as $n_{i+1}/n_i = 2$ by construction), the above inequality can be rearranged as
\begin{align*}
    \forall ~ i \geq 0,
    \qquad
    \frac{\Delta_{i+1}}{16 E_{i+1}}
    \leq \frac{\sqrt{\Delta_i E_i}}{16 E_{i+1}}
    = \frac{E_i}{16 E_{i+1}} \sqrt\frac{\Delta_i}{E_i}
    \leq \sqrt{\frac{\Delta_i}{16 E_i}}
    .
\end{align*}
This implies that for $k > \log \log (\Delta_1/E_1)$, it holds that $\Delta_k \leq 2 E_k$.
Observing that $\Delta_1 \leq 2L^2 \lambda$ (due to strong convexity) and $E_1 \geq 2L^2 \lambda / n$, we see that after $k = \ceil{\log\log{n}}$ phases, we hold a solution with error
\begin{align*}
    \E[F(w_k)] - F(w^\ast)
    \leq
    \frac{8 c^2 L^2}{\lambda} \lr{\frac 1{n} + \frac{d}{\rho^2 n^2}}
    .&\qedhere
\end{align*}
\end{proof}

\subsection{Direct Algorithms for the Strongly Convex Case}
\label{sec:strongly-convex}
\label{app:strongly-convex}
Now we show a linear time algorithm for the $\lambda$-strongly convex case, as long as the condition number $\kappa = \beta/\lambda$ is bounded by $O(n/\log n)$. Towards this goal, we first analyze a \emph{fixed} step-size algorithm for stochastic strongly convex optimization.%
\footnote{In typical variants of strongly convex (stochastic) gradient descent, one employs a decaying step-size schedule of the form $\eta_t = 1/(\lambda t)$ for obtaining the optimal convergence rate. Here we show that the same rate (up to a logarithmic factor) can be attained by a \emph{fixed} step-size algorithm, which is useful for our privacy analysis.}

\subsubsection{Fixed Step-size Algorithm for Strongly Convex SCO}

\begin{lem}
\label{lem:sgd-phase-sc}
Consider PSGD iterations with a \emph{fixed} step size $\eta$.
Suppose that $\eta \leq \frac{1}{2\lambda}$ and define weights $\gamma_t = (1-\eta\lambda)^{-t}$ for $t=1,\ldots,T$. Then for any $w$,
\begin{align*}
    \E\left[\frac{1}{\sum_{t=1}^T \gamma_t} \sum_{t=1}^T \gamma_tF(w_t)\right] - F(w) \leq \frac{\lambda}{e^{\eta\lambda T} - 1}  \norm{w_1-w}^2 + \frac{L^2}{2} \eta.
\end{align*}
In particular, setting $\eta = 2\log(T)/\lambda T$ ensures that the average
iterate
$
	\overline{w}_T = (\sum_{t=1}^T \gamma_t)^{-1} \sum_{t=1}^T \gamma_t w_t
$
has, for $T > 1$,
\begin{align*}
	\E[F(\overline{w}_T)] - F(w^\ast)
	\leq
	\frac{5L^2 \log{T}}{\lambda T}
	.
\end{align*}
\end{lem}

Observe that for ensuring $\eta \leq 1/\beta$, it is sufficient that $T =
\Omega(\kappa \log{\kappa})$ for $\kappa = \beta/\lambda$.

\begin{proof}
Denote the gradient vector used on iteration $t$ by $g_t$.
Following the standard SGD analysis, we can obtain
\begin{align*}
	g_t \cdot (w_t-w^\ast)
	\leq
	\frac{1}{2\eta} \lr{ \norm{w_t-w^\ast}^2 - \norm{w_{t+1}-w^\ast}^2 } + \frac{\eta}{2} \norm{g_t}^2
\end{align*}
for all $t$, and taking expectations of the above yields
\begin{align*}
	\E[\nabla F(w_t) \cdot (w_t-w^\ast)]
	\leq
	\frac{1}{2\eta} \lr{ \E\norm{w_t-w^\ast}^2 - \E\norm{w_{t+1}-w^\ast}^2 } + \frac{\eta}{2} \E\norm{g_t}^2
	.
\end{align*}
On the other hand, the $\lambda$-strong convexity of $F$ implies
\begin{align*}
	F(w_t) - F(w^\ast)
	\leq
	\nabla F(w_t) \cdot (w_t-w^\ast) - \frac{\lambda}{2} \norm{w_t-w^\ast}^2
	.
\end{align*}
Combining inequalities and summing over $t=1,\ldots,T$ with coefficients $\gamma_1,\ldots,\gamma_T$, we obtain
\begin{align*}
	&\sum_{t=1}^T \gamma_t \E[F(w_t) - F(w^\ast)]
	\\
	&\qquad\leq
	\sum_{t=1}^T \frac{\gamma_t}{2\eta} \lr{ \E\norm{w_t-w^\ast}^2 - \E\norm{w_{t+1}-w^\ast}^2 }
	- \frac{\lambda\gamma_t}{2} \sum_{t=1}^T \E\norm{w_t-w^\ast}^2
	+ \sum_{t=1}^T \frac{\eta\gamma_t}{2} \E\norm{g_t}^2
	\\
	&\qquad\leq
	\frac{\gamma_1}{2\eta} \norm{w_1-w^\ast}^2
	+\frac{1}{2} \sum_{t=2}^T \lr{\frac{\gamma_t-\gamma_{t-1}}{\eta}-\lambda\gamma_t} \norm{w_t-w^\ast}^2
	+ \frac{L^2\eta}{2} \sum_{t=1}^T \gamma_t
	,
\end{align*}
where in the final inequality we have used our assumption that $\E\norm{g_t}^2
\leq L^2$ for all~$t$.
Now, set $\gamma_t = (1-\eta\lambda)^{-t}$ in the bound above. Observe that
$\gamma_1/\eta \leq 2/\eta$ (as we required that $\eta \leq
1/2\lambda$), and for all $t>1$,
\begin{align*}
	\frac{\gamma_t-\gamma_{t-1}}{\eta}-\lambda\gamma_t
	=
	\frac{\gamma_t(1-\eta\lambda)-\gamma_{t-1}}{\eta}
	=
	0
	.
\end{align*}
Also, a simple computation shows that
 \begin{align*}
 	\sum_{t=1}^T \gamma_t
 	=
 	\frac{1}{\eta\lambda} \Lr{(1-\eta\lambda)^{-T}-1} 	
 	\geq
 	\frac{1}{\eta\lambda} \Lr{e^{\eta\lambda T}-1}.
 \end{align*}
We therefore obtain
\begin{align*}
	\sum_{t=1}^T \frac{\gamma_t}{\sum_{s=1}^T \gamma_s} \E[F(w_t) - F(w^\ast)]
	\leq
	\frac{\lambda}{e^{\eta\lambda T}-1} \norm{w_1-w^\ast}^2 + \frac{L^2}{2} \eta
	.
\end{align*}
By plugging in our choice of $\eta$ and applying Jensen's inequality on the
left-hand side, we establish the first bound. The second bound is obtained by
plugging in $\eta = \log(T)/\lambda T$ and bounding $\norm{w_1-w^\ast}^2 \leq
4L^2/\lambda^2$ (using strong convexity).
\end{proof}

\subsubsection{Direct Algorithm via Iterative Localization}

We can now analyze a variant of \cref{alg:sgd-phased} for the strongly convex
case, with appropriately chosen parameters.

\begin{thm} \label{thm:sgd-phased-sc}
Assume that in \cref{alg:sgd-phased}, we set $k = \ln \ln n$, $n_i = n/k$, $\eta_i = 2^{-2^i} \eta$  and $\eta = \frac{4ck\ln n}{\lambda n}$.
Then for the output of \cref{alg:sgd-phased}, we have
\begin{align*}
	\E[F(w_k)] - F(w^\ast)
	\leq
    O\left(\frac{L^2\ln n\ln\ln n}{\lambda}\lr{\frac{1}{n} + \frac{d\ln\ln n}{\rho^2n^2}}\right)
	,
\end{align*}
provided that $\kappa \leq \frac{n}{k\ln n}$.
\end{thm}
\begin{proof}
Denote $\overline{w}_0 = w^\ast$ and $\xi_0 = w_0-w^\ast$; by strong convexity, $\norm{\xi_0}_2 \leq 2L^2/\lambda^2$.
Using \cref{lem:sgd-phase-sc}, the total error of the algorithm can be bounded by
\begin{align*}
	\E[F(w_k)] - F(w^\ast)
	&=
	\sum_{i=1}^k \E[F(\overline{w}_i) - F(\overline{w}_{i-1})]
	+ \E[F(w_k) - F(\overline{w}_k)]
	\\
	&\leq
	\sum_{i=1}^k \lr{\frac{2\lambda}{e^{\eta_i\lambda n_i} - 1} \E[\norm{\xi_{i-1}}^2]
	+ \frac{\eta_i L^2}{2} } + L \cdot \E[\norm{\xi_k}_2]
	\\
	&\leq
	\frac{8\lambda}{e^{\eta_1\lambda n_1}}\frac{L^2}{\lambda^2} + \sum_{i=2}^k \lr{\frac{2\lambda}{\eta_i\lambda n_i}}\cdot \frac{dL^2\eta_{i-1}^2}{\rho^2}
	+ \frac{L^2}{2} \sum_{i=1}^k \eta_i
	+ L \cdot \E[\norm{\xi_k}_2].
\end{align*}
Here we have used the fact that $\exp(a)-1 \geq \exp(a)/2$ for $a\geq 1$ and $\exp(a)-1 \geq a$.
Continuing from above,
\begin{align*}
&\leq
	\frac{8L^2}{\lambda n^c} + \frac{2dkL^2}{\rho^2n}\sum_{i=2}^k \frac{\eta_{i-1}^2}{\eta_i} + \frac{L^2}{2} \sum_{i=1}^k  \eta_i  + L \cdot (\sqrt{d}\eta_k L)
	\\
	&\leq
	\frac{8L^2}{\lambda n^c} + \frac{2kdL^2\eta}{\rho^2 n} + \frac{L^2\eta}{2}  + \frac{L^2\sqrt{d}\eta}{2^{2^k}}
	\\
	&\leq
	\frac{8L^2}{\lambda n^c} + \frac{8ck^2dL^2\ln n}{\rho^2n^2} + \frac{4ckL^2\ln n}{\lambda n}   + \frac{ck L^2\ln n\sqrt{d}}{\lambda n^2}.
\end{align*}
For $c \geq 2$, it is easy to check that each of the terms is bounded by $\max\Lrset{O(\frac{kL^2\ln n}{\lambda n}), O(\frac{dk^2L^2\ln n}{\lambda \rho^2n^2})}$. The claim follows.
\end{proof}

\subsubsection{Direct Algorithm via Privacy Amplification by Iteration}

We next derive a variant of Snowball-SGD for the strongly convex case.

\begin{thm}
\label{thm:sc-privacy-utility}
Let $\K \subseteq \Re^d$ be a convex set of diameter $D$ and
$\{f(\cdot,x)\}_{x\in \cX}$ be a family of $\lambda$-strongly convex
$L$-Lipschitz and $\beta$-smooth functions over $\K$. For $T \in \Nat$, $\rho >
0$, and all $t\in [T]$ let $B_t = \lceil 2\sqrt{d/(T-t+1)}/\rho \rceil$, $n =
\sum_{t\in [T]} B_t$, $\eta = \frac{2 \log T}{\lambda T}$, $\sigma =
L/\sqrt{d}$, If $\eta \leq 2/\beta$ then for all $\alpha \geq 1$, starting point
$w_0\in \K$, and $S\in \cX^n$, PNSGD$(S,w_0,\{B_t\},\{\eta\},\{\sigma\})$
satisfies $ \left(\alpha, \alpha \cdot \rho^2/2 \right)$-RDP. Further, if $S$
consists of samples drawn i.i.d.~from a distribution $\cP$, then $n \leq T +
4\sqrt{d T}/\rho$ and
\begin{align*}
	\E[F(w_T)] 
	\leq F^* + \frac{20 L^2 \log^2 T}{\lambda T} 
	\leq F^* + \frac{40 L^2\log^2 n}{\lambda} 
		\cdot \left(\frac{1}{n} + \frac{16d}{\rho^2 n^2} \right)
	,
\end{align*}
where, for all $w\in \K$, $F(w) \doteq \E_{x \sim \cP}[f(w,x)]$, $F^* \doteq
\min_{w\in \K} F(w)$ and the expectation is taken over the random choice of $S$
and noise added by PNSGD.
\end{thm}

\begin{proof}
The privacy proof is identical to that of~\cref{thm:sz-privacy-utility}.  For
the utility analysis, we prove the following bound for the last iterate of our
fixed step-size algorithm:
\begin{align} \label{eq:sgd-bound-sc1}
	\E[F(w_T)] - \E\left[\frac{1}{\sum_{t=1}^T \gamma_t} \sum_{t=1}^T \gamma_tF(w_t)\right]
	\leq
	\frac{L^2}{2} \eta \sum_{k=1}^T \frac{\gamma_{T-k}}{\Gamma_{k-1}}
	.
\end{align}
To see this, define $S_k = \E\big[\brk{\sum_{t=T-k}^T \gamma_t}^{-1}
\sum_{t=T-k}^T \gamma_tF(w_t)\big]$ and recall from~\cref{lem:sgd-phase-sc} that
for any $k$, setting $w = w_{T-k}$ and viewing the algorithm as running SGD for
$T-k$ steps starting from $w_{T-k}$, we get:
\begin{align} \label{eq:sgd-bound-sc2}
	S_k
	&\leq 
	F(w_{T-k}) + \frac{L^2}{2} \eta
	.
\end{align}
We next relate $S_{k-1}$ and $S_k$. Let $\Gamma_k = \sum_{t=T-k}^T \gamma_t$.
Note that
\begin{align*}
    \Gamma_{k-1}S_{k-1} &= \Gamma_k S_k - \gamma_{T-k} \E[F(w_{T-k})]
	= 
	\Gamma_{k-1} S_k + \gamma_{T-k}(S_k - \E[F(w_{T-k})])
	.
\end{align*}
Dividing by $\Gamma_{k-1}$ and using~\cref{eq:sgd-bound-sc2}, we conclude
\begin{align*}
	S_{k-1} 
	&\leq 
	S_k + \frac{\gamma_{T-k}}{\Gamma_{k-1}} \cdot \frac{L^2}{2} \eta
	.
\end{align*}
Unravelling the recursion and observing that $\E[F(w_T)] = S_0$ yields
\cref{eq:sgd-bound-sc1}.
Using the fact that the~$\gamma_t$ are non-decreasing and applying
\cref{lem:sgd-phase-sc}, it follows that
\begin{align*}
	\E[F(w_T)]
	&\leq F^* + \frac{5L^2 \log T}{\lambda T} + \frac{L^2}{2} \eta \sum_{k=1}^T \frac{1}{k}
	\\
	&\leq F^* + \frac{10L^2 \log^2 T}{\lambda T}
	.
\end{align*}
The claimed utility bound follows from the fact that, as in the proof
of~\cref{thm:sz-privacy-utility}, the expected second moment of the gradient
goes up from $L^2$ to $2L^2$. The final bound follows by noting that
\begin{align*}
    n &\leq T + 4\sqrt{d T}/\rho \leq 2\max\{ T, 4\sqrt{d T}/\rho \}
    \\
    \implies T &\geq \frac 1 2 \min\left\{n, \frac{\rho^2n^2}{16d}\right\}
    \\
    \implies \frac 1 T &\leq 2 \max\left\{ \frac 1 n, \frac{16d}{\rho^2n^2} \right\}.
\end{align*}
The claim follows.
\end{proof}

\section{No Privacy Amplification by Averaged Iteration}
\label{sec:no-average-privacy}
\newcommand{\vx}{\mathbf{x}}
A common technique in convex optimization is to use iterate averaging. A plausible conjecture is that the average of the iterates enjoys privacy properties similar to the last iterate.
Indeed, in a Contractive Noisy Iteration with uniform noise, the privacy for the last iterate and that for the average iterate are within constant factors of each other when the contractive map is the identity.

Here we show that this does not hold true in general. Consider the contractive noise process defined by contractive maps:
\begin{align*}
\phi_i(\vx) = \left\{\begin{array}{ll} \vx &\mbox{if } i \leq k;\\ {\bf 0} & \mbox{otherwise.}\\ \end{array}\right.
\end{align*}
Here $k$ is a parameter we will set appropriately. Thus the contractive noise process is
\begin{align*}
X_{t+1} = \left\{\begin{array}{ll} X_t + \mathcal{N}({\bf 0}, \sigma^2) &\mbox{if } t \leq k;\\ \mathcal{N}({\bf 0}, \sigma^2) & \mbox{otherwise.}\\ \end{array}\right.
\end{align*}
The sum of $X_t$'s thus is easily seen to be distributed as:
\begin{align*}
  k X_0 + \sum_{i \leq k} (k-i+1)\mathcal{N}({\bf 0}, \sigma^2) + \sum_{i=k+1}^T \mathcal{N}({\bf 0}, \sigma^2).
\end{align*}
Simplifying, the average iterate is distributed as:
\begin{align*}
  \frac{k}{T} X_0 + \mathcal{N}\brk*{ {\bf 0}, \frac{O(k^3) + (T-k)}{T^2}\sigma^2 }.
\end{align*}
For $\sigma = \frac{1}{\sqrt{T}}$, where the final iterate has $(\alpha, O(\alpha))$-RDP, this simplifies to
\begin{align*}
  \frac{k}{T} X_0 + \mathcal{N}\brk*{ {\bf 0}, \frac{O(k^3) + (T-k)}{T^2} \sigma^2 }
  &= \frac{k}{T} \left( X_0 + \mathcal{N}\brk*{ {\bf 0}, O\brk*{\frac{k}{T}} + \frac{(T-k)}{Tk^2} }\right).
\end{align*}
Whereas for $k=1$ and for $k=T$, this amount of noise gives $(\alpha, O(\alpha))$-RDP, for intermediate values of $k$, e.g., $k \in [T^{\frac{1}{3}}, T^{\frac{2}{3}}]$, the effective amount of noise is not sufficient to mask $X_0$.

A similar lower bound can be realized for online convex optimization. Consider the sequence of loss functions over $\Re$ defined as:
\begin{align*}
    \ell_t(w) = \left\{
    \begin{array}{ll}
    (w - b)^2 &t=1;\\
    0 &2 \leq t \leq k;\\
    w^2 & k+1 \leq t \leq T.\\
    \end{array}
    \right.
\end{align*}
Here $k$ is a parameter to be set appropriately, and $b\in \{-1, 1\}$. Suppose that step size is $\eta = \frac{1}{\sqrt{T}}$ and the noise scale at each step is $\frac{1}{\sqrt{T}}$. If the noise added to the gradient at step $t$ is $\xi_t$, then one can verify that
the average iterate is
\begin{align*}
    \frac{k}{T} (b\eta) + \sum_{t \leq k} (O(\eta) + (k-t)) \eta\xi_t + \sum_{t=k+1}^T O(\eta) \eta\xi_t.
\end{align*}
In other words, the average iterate is distributed as
\begin{align*}
    \frac{k\eta}{T} \left(b + \mathcal{N}\brk*{0, O\brk*{\frac{k}{T}} + \frac{(T-k)}{Tk^2} }\right).
\end{align*}
This is the same behaviour as in the counterexample above. Thus the average is not $(\alpha, O(\alpha))-RDP$. This example can be easily modified to handle suffix averaging over a $\Omega(T)$-sized suffix.  \bibliographystyle{alpha}
\bibliography{dpsco-refs,bib}

\newcommand{\etalchar}[1]{$^{#1}$}
\begin{thebibliography}{DKM{\etalchar{+}}06}

\bibitem[ACG{\etalchar{+}}16]{DLDP}
Mart{\'{\i}}n Abadi, Andy Chu, Ian~J. Goodfellow, H.~Brendan McMahan, Ilya
  Mironov, Kunal Talwar, and Li~Zhang.
\newblock Deep learning with differential privacy.
\newblock In {\em Proceedings of the 2016 {ACM} {SIGSAC} Conference on Computer
  and Communications Security (CCS)}, pages 308--318, 2016.

\bibitem[BE02]{BousquetElisseeff-2002}
Olivier Bousquet and Andr\'e Elisseeff.
\newblock Stability and generalization.
\newblock {\em JMLR}, 2002.

\bibitem[BFTT19]{BassilyFTT19}
Raef Bassily, Vitaly Feldman, Kunal Talwar, and Abhradeep Thakurta.
\newblock Private stochastic convex optimization with optimal rates.
\newblock {\em CoRR}, abs/1908.09970, 2019.
\newblock Extended abstract in Proceedings of {NeurIPS} 2019.

\bibitem[BS16]{BS16-zCDP}
Mark Bun and Thomas Steinke.
\newblock Concentrated differential privacy: Simplifications, extensions, and
  lower bounds.
\newblock In {\em Theory of Cryptography---14th International Conference, {TCC}
  2016-B, Part {I}}, pages 635--658, 2016.

\bibitem[BST14]{bassily2014differentially}
Raef Bassily, Adam Smith, and Abhradeep Thakurta.
\newblock Private empirical risk minimization: Efficient algorithms and tight
  error bounds.
\newblock In {\em Foundations of Computer Science (FOCS), 2014 IEEE 55th Annual
  Symposium on}, pages 464--473. IEEE, 2014.

\bibitem[CM08]{CM08}
Kamalika Chaudhuri and Claire Monteleoni.
\newblock Privacy-preserving logistic regression.
\newblock In Daphne Koller, Dale Schuurmans, Yoshua Bengio, and L{\'e}on
  Bottou, editors, {\em NIPS}. MIT Press, 2008.

\bibitem[CMS11]{chaudhuri2011differentially}
Kamalika Chaudhuri, Claire Monteleoni, and Anand~D Sarwate.
\newblock Differentially private empirical risk minimization.
\newblock {\em Journal of Machine Learning Research}, 12(Mar):1069--1109, 2011.

\bibitem[DJW13]{DuchiJW13}
John~C. Duchi, Michael~I. Jordan, and Martin~J. Wainwright.
\newblock Local privacy and statistical minimax rates.
\newblock In {\em IEEE 54th Annual Symposium on Foundations of Computer Science
  (FOCS)}, pages 429--438, 2013.

\bibitem[DKM{\etalchar{+}}06]{DKMMN06}
Cynthia Dwork, Krishnaram Kenthapadi, Frank McSherry, Ilya Mironov, and Moni
  Naor.
\newblock Our data, ourselves: Privacy via distributed noise generation.
\newblock In {\em EUROCRYPT}, 2006.

\bibitem[DMNS06]{DMNS06}
Cynthia Dwork, Frank McSherry, Kobbi Nissim, and Adam Smith.
\newblock Calibrating noise to sensitivity in private data analysis.
\newblock In {\em Theory of Cryptography Conference}, pages 265--284. Springer,
  2006.

\bibitem[DR14]{DR14-book}
Cynthia Dwork and Aaron Roth.
\newblock The algorithmic foundations of differential privacy.
\newblock {\em Found. Trends Theor. Comput. Sci.}, 9(3--4):211--407, August
  2014.

\bibitem[DR16]{DworkRothblum-CDP}
Cynthia Dwork and Guy~N. Rothblum.
\newblock Concentrated differential privacy.
\newblock {\em CoRR}, abs/1603.01887, 2016.

\bibitem[Dwo06]{Dwork06}
Cynthia Dwork.
\newblock Differential privacy.
\newblock In {\em ICALP}, 2006.

\bibitem[Fel16]{feldman2016generalization}
Vitaly Feldman.
\newblock Generalization of erm in stochastic convex optimization: The
  dimension strikes back.
\newblock In {\em Advances in Neural Information Processing Systems}, pages
  3576--3584, 2016.

\bibitem[FMTT18]{FeldmanMTT18}
Vitaly Feldman, Ilya Mironov, Kunal Talwar, and Abhradeep Thakurta.
\newblock Privacy amplification by iteration.
\newblock {\em CoRR}, abs/1808, 2018.
\newblock Extended abstract in Proceedings of {FOCS} 2018.

\bibitem[FV19]{feldman2019high}
Vitaly Feldman and Jan Vondrak.
\newblock High probability generalization bounds for uniformly stable
  algorithms with nearly optimal rate.
\newblock {\em arXiv preprint arXiv:1902.10710}, 2019.

\bibitem[Har19]{Harvey:19pc}
Nicholas Harvey.
\newblock Personal communication, 2019.

\bibitem[HK14]{hazan2014beyond}
Elad Hazan and Satyen Kale.
\newblock Beyond the regret minimization barrier: optimal algorithms for
  stochastic strongly-convex optimization.
\newblock {\em The Journal of Machine Learning Research}, 15(1):2489--2512,
  2014.

\bibitem[HLPR19]{HarveyLPR19}
Nicholas J.~A. Harvey, Christopher Liaw, Yaniv Plan, and Sikander Randhawa.
\newblock Tight analyses for non-smooth stochastic gradient descent.
\newblock In {\em {COLT}}, pages 1579--1613, 2019.

\bibitem[HRS15]{hardt2015train}
Moritz Hardt, Benjamin Recht, and Yoram Singer.
\newblock Train faster, generalize better: Stability of stochastic gradient
  descent.
\newblock {\em arXiv preprint arXiv:1509.01240}, 2015.

\bibitem[INS{\etalchar{+}}19]{iyengartowards}
Roger Iyengar, Joseph~P Near, Dawn Song, Om~Thakkar, Abhradeep Thakurta, and
  Lun Wang.
\newblock Towards practical differentially private convex optimization.
\newblock In {\em IEEE S and P (Oakland)}, 2019.

\bibitem[JKT12]{jain2012differentially}
Prateek Jain, Pravesh Kothari, and Abhradeep Thakurta.
\newblock Differentially private online learning.
\newblock In {\em 25th Annual Conference on Learning Theory (COLT)}, pages
  24.1--24.34, 2012.

\bibitem[JNN19]{JainNN19}
Prateek Jain, Dheeraj Nagaraj, and Praneeth Netrapalli.
\newblock Making the last iterate of {SGD} information theoretically optimal.
\newblock In {\em {COLT}}, pages 1752--1755, 2019.

\bibitem[JT14]{JTOpt13}
Prateek Jain and Abhradeep Thakurta.
\newblock (near) dimension independent risk bounds for differentially private
  learning.
\newblock In {\em ICML}, 2014.

\bibitem[KST12]{kifer2012private}
Daniel Kifer, Adam Smith, and Abhradeep Thakurta.
\newblock Private convex empirical risk minimization and high-dimensional
  regression.
\newblock In {\em Conference on Learning Theory}, pages 25--1, 2012.

\bibitem[Mir17]{mironov2017renyi}
Ilya Mironov.
\newblock {\Renyi} differential privacy.
\newblock In {\em 30th {IEEE} Computer Security Foundations Symposium (CSF)},
  pages 263--275, 2017.

\bibitem[Nes04]{nesterov-book}
Yurii Nesterov.
\newblock {\em Introductory Lectures on Convex Optimization. A Basic Course.}
\newblock Springer US, 2004.

\bibitem[PAE{\etalchar{+}}17]{PATE}
Nicolas Papernot, Mart{\'\i}n Abadi, {\'U}lfar Erlingsson, Ian Goodfellow, and
  Kunal Talwar.
\newblock Semi-supervised knowledge transfer for deep learning from private
  training data.
\newblock In {\em Proceedings of the 5th International Conference on Learning
  Representations (ICLR)}, 2017.

\bibitem[PSM{\etalchar{+}}18]{PATE2}
Nicolas Papernot, Shuang Song, Ilya Mironov, Ananth Raghunathan, Kunal Talwar,
  and {\'U}lfar Erlingsson.
\newblock Scalable private learning with {PATE}.
\newblock In {\em Proceedings of the 6th International Conference on Learning
  Representations (ICLR)}, 2018.

\bibitem[R{\'e}n61]{Renyi61}
Alfr\'ed R{\'e}nyi.
\newblock On measures of entropy and information.
\newblock In {\em Proceedings of the fourth Berkeley symposium on mathematical
  statistics and probability}, volume~1, pages 547--561, 1961.

\bibitem[SCS13]{song2013stochastic}
Shuang Song, Kamalika Chaudhuri, and Anand~D Sarwate.
\newblock Stochastic gradient descent with differentially private updates.
\newblock In {\em IEEE Global Conference on Signal and Information Processing},
  2013.

\bibitem[SSSSS09]{SSSS}
Shai Shalev-Shwartz, Ohad Shamir, Nathan Srebro, and Karthik Sridharan.
\newblock {Stochastic Convex Optimization}.
\newblock In {\em COLT}, 2009.

\bibitem[SSSSS10]{shalev2010learnability}
S.~Shalev-Shwartz, O.~Shamir, N.~Srebro, and K.~Sridharan.
\newblock Learnability, stability and uniform convergence.
\newblock {\em JMLR}, 2010.

\bibitem[ST13]{ST13sparse}
Adam Smith and Abhradeep Thakurta.
\newblock Differentially private feature selection via stability arguments, and
  the robustness of the {LASSO}.
\newblock In {\em Conference on Learning Theory (COLT)}, pages 819--850, 2013.

\bibitem[STU17]{smith2017interaction}
Adam Smith, Abhradeep Thakurta, and Jalaj Upadhyay.
\newblock Is interaction necessary for distributed private learning?
\newblock In {\em IEEE Security \& Privacy}, pages 58--77, 2017.

\bibitem[SZ13]{shamir2013stochastic}
Ohad Shamir and Tong Zhang.
\newblock Stochastic gradient descent for non-smooth optimization: Convergence
  results and optimal averaging schemes.
\newblock In {\em ICML}, pages 71--79, 2013.

\bibitem[TTZ15]{talwar2015nearly}
Kunal Talwar, Abhradeep Thakurta, and Li~Zhang.
\newblock Nearly optimal private {LASSO}.
\newblock In {\em Proceedings of the 28th International Conference on Neural
  Information Processing Systems}, volume~2, pages 3025--3033, 2015.

\bibitem[Ull15]{ullman2015private}
Jonathan Ullman.
\newblock Private multiplicative weights beyond linear queries.
\newblock In {\em Proceedings of the 34th ACM SIGMOD-SIGACT-SIGAI Symposium on
  Principles of Database Systems}, pages 303--312. ACM, 2015.

\bibitem[WLK{\etalchar{+}}17]{wu2017bolt}
Xi~Wu, Fengan Li, Arun Kumar, Kamalika Chaudhuri, Somesh Jha, and Jeffrey
  Naughton.
\newblock Bolt-on differential privacy for scalable stochastic gradient
  descent-based analytics.
\newblock In {\em SIGMOD}. ACM, 2017.

\bibitem[WYX17]{wang2017differentially}
Di~Wang, Minwei Ye, and Jinhui Xu.
\newblock Differentially private empirical risk minimization revisited: Faster
  and more general.
\newblock In {\em Advances in Neural Information Processing Systems}, pages
  2722--2731, 2017.

\end{thebibliography}

\end{document}